\documentclass{article}

\usepackage{microtype}
\usepackage{graphicx}
\usepackage{subfigure}
\usepackage{booktabs} %

\usepackage{hyperref}

\usepackage[accepted]{icml2020}

\icmltitlerunning{Survey bandits}

\usepackage{amsmath,amsthm,amssymb}
\usepackage{cleveref}
\usepackage{thmtools,thm-restate}
\usepackage{caption}

\newcommand{\suchthat}{\;\ifnum\currentgrouptype=16 \middle\fi|\;}
\newcommand{\F}{\mathcal{F}} 
\newcommand{\bmin}{\beta_{\text{min}}}
\newcommand{\optimisticEst}{\Tilde{\beta}}
\newcommand{\trueEst}{\hat{\beta}}
\newcommand{\Vtmatrix}{V_{t-1}^i}
\newcommand{\VtmatrixInv}{(V_{t-1}^i)^{-1}}
\newcommand{\X}{\mathbf{X}}
\newcommand{\conftBound}[2]{\theta_{#1,#2}} %

\newcommand{\cset}[2]{C_{#1,#2}} %

\newcommand{\Dtmatrix}[2]{D_{#2}^{#1}}

\newcommand{\Supp}{\text{Supp}}
\newcommand{\SurveyUCB}{\text{SurveyUCB}}
\newcommand{\AlgConfidence}{\text{AlgConfidence}}
\newcommand{\transpose}{\mathsf{T}} %
\newcommand{\trace}{\text{trace}}
\newcommand{\R}{\mathbb{R}}
\newcommand{\ridge}{\text{ridge}}
\newcommand{\elnet}{\text{elnet}}
\newcommand\blfootnote[1]{%
  \begingroup
  \renewcommand\thefootnote{}\footnote{#1}%
  \addtocounter{footnote}{-1}%
  \endgroup
}

\newtheorem{assumption}{Assumption}
\newtheorem{definition}{Definition}
\newtheorem{lemma}{Lemma}

\begin{document}

\twocolumn[
\icmltitle{Survey Bandits with Regret Guarantees}

\icmlsetsymbol{equal}{*}

\begin{icmlauthorlist}
\icmlauthor{Sanath Kumar Krishnamurthy}{to}
\icmlauthor{Susan Athey}{to}
\end{icmlauthorlist}

\icmlaffiliation{to}{Stanford University}

\icmlcorrespondingauthor{Sanath Kumar Krishnamurthy}{sanathsk@stanford.edu}

\icmlkeywords{Machine Learning, Contextual bandits}

\vskip 0.3in
]

\blfootnote{$^1$ Stanford University\\
The first author is supported in part by a Dantzig-Lieberman Operations Research Fellowship and a Management Science and Engineering Department Fellowship.\\
The second author is supported in part by Sloan Foundation, Schmidt Futures, and ONR grant N00014-17-1-2131.}

\begin{abstract}
We consider a variant of the contextual bandit problem. In standard contextual bandits, when a user arrives we get the user’s complete feature vector and then assign a treatment (arm) to that user. In a number of applications (like healthcare), collecting features from users can be costly. To address this issue, we propose algorithms that avoid needless feature collection while maintaining strong regret guarantees.
\end{abstract}

\section{Introduction}

In a number of applications, like healthcare and mobile health, collecting features from users can be costly. This encourages us to develop variants of contextual bandits that at any time-step select a set of features to be collected from users, and then select an action based on these observed features. We use the term survey bandits to refer to this set-up. Through this formulation, we address the issue of needless feature collection in contextual bandits. 

Suppose we are building a system to recommend charities that users can donate to. Since there are many charities we could recommend, it is efficient to make personalized recommendations. This requires us to collect features from users, which can be done by requiring users to fill a survey/questionnaire. The reward at any time-step might be the amount donated. As usual, our goal is to minimize regret. Beyond regret, we would also like to improve user experience by shortening the survey we require users to fill. We try to answer the following question: Can we ensure strong regret guarantees, while being able to shorten our survey over time? %

\textbf{Our contributions:}
We answer this question in the affirmative. We start by considering zero-shot surveys where the decision maker has to decide the set of features to be queried at every time-step before the user arrives and then make a personalized decision based on the responses. We now state our assumptions. In addition to the standard assumptions made for LinUCB, we introduce \ref{ass:beta-min} (beta-min) which is common in the feature selection literature. We propose algorithms that are natural variants of LinUCB \cite{li2010contextual} for the survey bandits framework. Under our assumptions, we prove regret guarantees for these algorithms. At the same time, these algorithms exploit sparsity in arm parameters to reduce the set of features collected from users.  

In fact our algorithm, Ridge $\SurveyUCB$ has regret guarantees that are tight even for standard contextual bandits $O(\sqrt{T}\log(T))$. We also provide an algorithm Elastic Net $\SurveyUCB$ which is more robust to \cref{ass:beta-min}, but with a weaker $O((T\log(T))^{3/4})$ regret guarantee. This result requires us to prove a new adaptive tail bounds for the elastic net estimator, which may be of independent interest.

Through simulations, when \cref{ass:beta-min} holds, we demonstrate that both algorithms perform well in terms of regret. Unfortunately, Ridge $\SurveyUCB$ can perform poorly on regret when \cref{ass:beta-min} is violated. Fortunately, Elastic Net $\SurveyUCB$ performs well even when \cref{ass:beta-min} is violated. It is worth noting that we can still use Ridge $\SurveyUCB$, we just need to use a conservative choices for the beta-min parameter (in \cref{ass:beta-min}). Even for conservative choices of the beta-min parameter, we  eventually see benefits on survey length.

Through simulations, we also see that both algorithms demonstrate savings in terms of survey length. In fact, in the absence of sub-optimal arms, both algorithms always remove the features that are not relevant for reward prediction in the survey.

We also consider settings with interactive survey's where at each time step, before making a personalized decision, the decision maker can continually expand the set of queries based on previous responses by the user at that time-step. This allows us to start by querying a smaller set of features and expand our survey for the user only if it is needed to make a better personalized decision. We develop variants of our algorithms that use interactive surveys to ensure lower survey lengths, especially in the presence of sub-optimal arms, while maintaining the same regret performance.

\textbf{Related work:} Prior work (e.g. \cite{abbasi2012online} and \cite{bastani2015online}) has also exploited sparsity in arm parameters to provide stronger sparsity dependent regret guarantees. When an upper-bound on the sparsity of arm-parameters is known to the decision-maker, \citet{abbasi2012online} provide algorithms with tight regret guarantees for contextual bandits. Under distributional assumptions on user covariates, \citet{bastani2015online} provide stronger regret guarantees, even when no sparsity parameter is provided to the algorithm. While these regret guarantees are stronger than the ones we provide, it is important to note that we do not make any distributional assumption on user covariates and we do not assume that the decision-maker knows an upper-bound to the sparsity of arm-parameters. 

\citet{bouneffouf2017context} is the most closely related paper to our work. They develop interesting algorithms for a very similar setup and evaluate them empirically. At every time-step, their algorithms queries a fixed number of features ($s$). Their algorithms requires this parameter $s$ to be an input to the algorithm. For a conservative choice of $s$, we would see little benefits to the survey length. Unfortunately, empirically their algorithms could perform much worse than contextual bandits for small choices of $s$. We view our work as an alternate approach with guarantees on regret. %

\section{Preliminaries}

\subsection{Problem Setting}

Let $T$ denote the number of time-steps. At each time-step, a new user arrives. The decision maker has a survey with $d$ questions and can ask a user any subset of these questions. At any time-step $t$, a user comes with a set of observable covariates $X_t$. Here $X_t$ is a $d$-vector, corresponding to the $t$-th user's observable answers to the $d$ questions on the survey.

The decision maker has access to $K$ arms (decisions). Each arm yields a noisy, user specific reward. In particular, each arm $i\in [K]$ has an unknown parameter $\beta_i\in \R^d$. At time $t$, pulling arm $i$ would yield a reward $Y_i(t):=X_t^{\transpose}\beta_i +\epsilon_{i,t}$. Where $\epsilon_{i,t}$ are independent sequence of $\sigma$-sub-Gaussian random variables. Note that at any time-step, the decision maker can only observe the reward of the arm that was pulled.

\textbf{Some notation:} For any vector $z \in \R^d$ and any index set $I\subseteq [d]$, we let $z_I \in \R^d$ denote the vector obtained by setting coordinates of $z$ that are not in $I$ to zero. For any matrix $A\in\R^{d\times d}$ and any index set $I\subseteq [d]$, we let $A_I \in \R^{d\times d}$ denote the matrix obtained by setting rows and columns of $A$ that do not correspond to an index in $I$ to zero. For any $z\in\R^d$, we let $\Supp(z)$ to denote the support. And for any set $\zeta\subseteq \R^d$, we let $\Supp(\zeta):=\cup_{z\in\zeta}\Supp(z)$. %

\textbf{Goal:} The goal is to design a sequential decision making policy $\pi=(\pi^s,\pi^a)$ that maximizes expected cumulative reward, and subject to strong reward guarantees minimizes the total number of questions asked to users. %
Let $\pi^s_t \in 2^{[d]}$ denote the subset of survey questions queried by policy $\pi$ at time $t \in [T]$. And, let $\pi^a_t \in [K]$ denote the arm chosen by policy $\pi$ at time $t \in [T]$. Note that we do not observe $X_t$ and only observe $(X_t)_{\pi^s_t}$, hence we should be able to choose the arm $\pi^a_t$ using only the observed covariates $(X_t)_{\pi^s_t}$ and data collected from previous time-steps. 

\textbf{Target policy:} We now describe a "sensible" target policy. Consider the target policy $\pi^*=(\pi^{s*}, \pi^{a*})$ that already knows arm parameters $\{\beta_i\}_{\{i\in[K]\}}$ but not the noise parameters. We want the target policy to maximize expected cumulative reward, and hence at any time-step $t$ the policy must pick an arm $\pi^{a*}_t \in \arg\max_j X_t^{\transpose}\beta_j$. Therefore, the target policy only needs to query features that influence arm rewards. That is, 
$$ \pi_t^{s*} := \bigcup_{i \in [K]} \pi^{s*}_{i,t} \text{, where $\pi^{s*}_{i,t} := \Supp(\beta_i)$}. $$
Note that for any vector $z\in\R^d$ and any arm~$i$, we have that $ z^{\transpose}\beta_i = z_{\pi^{s*}_t}^{\transpose}\beta_i$. Therefore at any time-step $t$, having observed $(X_t)_{\pi^{s*}_t}$, the target policy is able to choose the best arm for the covariate $X_t$:
$$  \pi^{a*}_t \in \arg\max_j X_t^{\transpose}\beta_j = \arg\max_j (X_t)_{\pi_t^{s*}}^{\transpose} \beta_j. $$

\textbf{Regret:} If $\pi^a_t=i$, we define expected regret at time $t$ as the difference between the maximum expected reward and the expected reward of arm~$i$ at time $t$, i.e. $r_t := \max_j X_t^{\transpose}\beta_j - X_t^{\transpose}\beta_i$. And, we define the cumulative expected regret as $R_T := \sum_{t=1}^T r_t$. 

\textbf{Additional notation:} Let $\X\in\R^{T\times d}$ denote the design matrix, whose rows are $X_t$. Let $Y_i \in \R^T$ denote the vector of observations $X_t^{\transpose}\beta_i + \epsilon_{i,t}$, where entries of $Y_i$ may be missing \footnote{If arm~$i$ wasn't played at time $t$, then the $t$-th coordinate of $Y_i$ would be missing.}. For all $i\in[K]$ and for any $n\in [T]$, define the sample set $S_{i,n}:=\{t| \pi_t^a=i, t\leq n \}$ and let $S_i:=\{t| \pi_t^a=i\} \subseteq [T]$. Let $n_{i,t}$ denote the number of times arm~$i$ was pulled upto time $t$ (i.e. $n_{i,t}=|S_{i,t}|$). For any $S' \subseteq [T]$, we let $\X(S')$ be the $|S'|\times d$ submatrix of $\X$ whose rows are $X_t$ for each $t\in S'$. Similarly when $S'\subseteq S_i$, we let $Y_i(S')$ be the $|S'|$-vector whose coordinates are $Y_i(t)$ for all $t\in S'$ \footnote{Since when $S'\subseteq S_i$, we know that $\pi_t^a=i$ for all $t\in S'$. Therefore, we have $Y_i(S')$ has no missing values.}. Also let $I\in\R^{d\times d}$ denote the identity matrix. %

\subsection{Assumptions}

We make two assumptions. The following assumption is to allow us to ignore features that have small influence on arm rewards, we assume that such features in-fact have no influence on arm rewards.
\begin{assumption}[Beta-min]
\label{ass:beta-min} 
The decision maker knows a parameter $\bmin \in \R_+$ such that for all arms $i \in [K]$ and all $q\in[d]$, either $\|(\beta_i)_{\{q\}}\|_1 = 0$ or $\|(\beta_i)_{\{q\}}\|_1>\bmin$. 
\end{assumption}

The following is a common assumption made in problems for contextual multi-arm bandits, it is equivalent to assuming that expected rewards are bounded.

\begin{assumption}[Bounded rewards]
\label{ass:bounded-rewards}
We make the following assumptions to ensure that expected rewards of any arm for any context is bounded: For all $t\in[T]$ and $i\in[K]$, we have $\|X_t\|_{1}\leq L$ and $\|\beta_i\|_1 \leq b$.\\
For simplicity we further assume for all $t\in[T]$ and $i\in[K]$, the potential reward of pulling arm~$i$ lies in $[0,1]$. i.e. $Y_i(t)\in[0,1]$. \footnote{We can avoid the assumption that $Y_i(t)\in[0,1]$ for all $i\in[K]$ and $t\in[T]$. We would just need to choose $\alpha \geq \max(1,L)$ in \cref{def:standard-form-confidence} for the regret guarantee to work out. Where $\|X_t\|_2\leq L$ for all $t$.}
\end{assumption}

Since for any vector $v\in\R^d$, we know that $\|v\|_p\leq \|v\|_1$ for all $p\geq 1$. Therefore for all $p\geq 1$, \cref{ass:bounded-rewards} gives us that the $p$-norms of $X_t$ and $\beta_i$ are bounded by $L$ and $b$ for all time-steps $t$ and arms $i\in[K]$. 

\section{UCB for Survey Bandits}

In this section, we describe a natural extension of LinUCB \cite{li2010contextual} for survey bandits which involves describing a policy for selecting survey questions and a consistent arm selection policy that can choose an arm given the observed covariates.

\begin{algorithm}
    \caption{$\SurveyUCB$}
    \label{alg:survey-ucb}
    \begin{algorithmic}[1] %
    \STATE Initialize confidence sets. (See \cref{subsec:standard-form})
    \FOR{$t:=1,2,\dots$}
        \STATE Let $\pi^s_{k,t}\leftarrow \Supp(\cset{k}{t-1})$ for all $k\in[K]$.
        \STATE Query $\pi^s_t:=\bigcup_{k\in[K]} \pi^s_{k,t}$, and observe $(X_t)_{\pi^s_t}$.
        \STATE Pick $\pi^a_t \in \arg\max_{k\in[K]} \max_{\beta \in \cset{k}{t-1}} (X_t)_{\pi^s_t}^{\transpose}\beta $.
        \STATE Play arm~$\pi^a_t$ and observe reward $Y_t$.
        \STATE Set $\cset{k}{t}\leftarrow \cset{k}{t-1}$ for all $k\in[K]\setminus\{\pi^a_t\}$.
        \STATE Update $\cset{\pi^a_t}{t}$ (using $\AlgConfidence$).
    \ENDFOR
    \end{algorithmic}
\end{algorithm}

Upper confidence bound (UCB) algorithms follow the principle of optimism in the face of uncertainty. The essential idea is to construct high-probability confidence sets $\cset{i}{t-1}\subseteq \R^d$, for the parameter $(\beta_i)$ of every arm $i\in[K]$, from observed covariates $(X_1)_{\pi^s_1},(X_2)_{\pi^s_2},\dots,(X_{t-1})_{\pi^s_{t-1}}$ and rewards $Y_1,Y_2,\dots,Y_{t-1}$. That is, with high probability, $\beta_i\in\cset{i}{t-1}$ for all time-steps $t$ and for all arms $i\in[K]$. The algorithm then queries the set of features: 
$$\pi^s_t := \bigcup_{i\in[K]} \pi^s_{i,t}, \text{ with } \pi^s_{i,t}=\Supp(\cset{i}{t-1}) \text{ for all } i\in[K].$$ 
Therefore, we have that $\beta=\beta_{\pi^s_t}$ for $\beta\in\bigcup_{i\in[K]}\cset{i}{t-1}$. Hence, for any $z\in\R^d$ and $\beta\in \bigcup_{i\in[K]}\cset{i}{t-1}$, we have that $z^{\transpose}\beta=z_{\pi^s_t}^{\transpose}\beta$. Therefore, it follows that:
$$ (X_t)_{\pi^s_t}^{\transpose}\beta = X_t^{\transpose}\beta, \text{ for all } \beta\in\bigcup_{i\in[K]}\cset{i}{t-1}. $$
The algorithm chooses an optimistic estimate $\optimisticEst_{i,t} \in \arg\max_{\beta\in \cset{i}{t-1}} X_t^{\transpose}\beta$ for every arm $i\in [K]$ and then chooses an arm $\pi_t^a \in \arg \max_{i\in[K]} X_t^{\transpose}\optimisticEst_{i,t}$ which maximizes reward according to the optimistic estimates. Equivalently, and more compactly, the algorithm chooses the arm:
\begin{align*}
    \pi_t^a &\in \arg\max_{k\in[K]} \max_{\beta \in \cset{k}{t-1}} X_t^{\transpose}\beta\\
    & = \arg\max_{k\in[K]} \max_{\beta \in \cset{k}{t-1}} (X_t)_{\pi^s_t}^{\transpose}\beta .
\end{align*}
Note that the arm $\pi^s_t$ is chosen given only the observed covariates $(X_t)_{\pi^s_t}$. We call the resulting algorithm $\SurveyUCB$.

\section{Confidence Sets for $\SurveyUCB$} \label{sec:confidence_sets}
In this section, we define confidence sets in Standard form and describe the construction of these sets using $\AlgConfidence$. It will turn-out that $\AlgConfidence$ constructs confidence sets in Standard form. Throughout this manuscript, we use $\AlgConfidence$ to construct confidence sets for $\SurveyUCB$.

\subsection{Confidence Sets in Standard Form}
\label{subsec:standard-form}
We start by defining weighted norms and use that to define confidence sets in Standard form.

\begin{definition}[Weighted norm]
For any vector $z\in\R^d$ and any positive semi-definite matrix $A$, we define the weighted norm of $z$ with respect to $A$ as follows: $\|z\|_A:=\sqrt{z^{\transpose}A z}$. \footnote{$\|\cdot\|_A$ is a norm when $A$ is positive semi-definite.}
\end{definition}

\begin{definition}[Standard form]
\label{def:standard-form-confidence}
We say confidence sets $\{\cset{i}{t-1}\}_{(i,t)\in[K]\times[T]}$ are in Standard form if at any time-step $t$ and for any arm $i\in[K]$, we have that $\cset{i}{t-1}$ is a ball centered around our estimate $(\trueEst_{i,t-1})$ of the true arm parameter $(\beta_i)$ under a weighted norm and $\{\cset{i}{t-1}\}_{t\in[T]}$ have non-increasing supports. More specifically, for all $i\in[K]$, $t\in [T]$ and for some $\alpha>0$, we have:
\begin{align*}
    \Supp(\trueEst_{i,t-1})\subseteq H_{i,t-1} := \Supp(\cset{i}{t-1})&\\
    H_{i,t}\subseteq H_{i,t-1}\subseteq [d], \text{ } 1\leq \conftBound{i}{t-1} \leq \conftBound{i}{t}&\\
    \Dtmatrix{i}{t-1} := (\alpha I)_{H_{i,t-1}} + \sum_{w\in S_{i,t-1}} (X_w)_{H_{i,t-1}} (X_w)_{H_{i,t-1}}^{\transpose}&\\
    \cset{i}{t-1}:=\{\beta| \text{ } \|\beta - \trueEst_{i,t-1} \|_{\Dtmatrix{i}{t-1}} \leq \conftBound{i}{t-1},\\
    \Supp(\beta)\subseteq H_{i,t-1} \}&
\end{align*}
\end{definition}

In $\SurveyUCB$, note that the confidence sets determine the set of features queried. Also at any time-step $t$, confidence sets must be constructed using only the observed covariates and rewards at every time-step upto $t$. Lemma~\ref{lem:observed-covariates} gives us a set of observed covariates when $\SurveyUCB$ uses confidence sets in Standard form.

\begin{lemma}\label{lem:observed-covariates}
If $\SurveyUCB$ uses confidence sets in Standard form for times $t\in\{1,2,\dots,t'\}$. Then for any arm $k\in[K]$, we observe $\{(X_t)_{H_{k,t'-1}}\}_{t=1}^{t'}$. Where $H_{k,t-1} := \Supp(\cset{k}{t-1})$ for all arms $k\in[K]$ and time-steps $t\in[t']$.
\end{lemma}
\begin{proof}
All statements in this proof hold for all arms $k\in[K]$ and time-steps $t\in[T]$. $\SurveyUCB$ queries the set of features $\pi^s_t=\cup_k\pi^s_{k,t}$ at time $t$, where $\pi^s_{k,t}=\Supp(\cset{k}{t-1})=H_{k,t-1}$. From the structure of confidence sets in Standard form, we have that $\pi^s_{k,t+1}=H_{k,t}\subseteq H_{k,t-1}=\pi^s_{k,t}$. This implies that: 
$$H_{k,t-1} = \pi^s_{k,t} \subseteq \pi^s_{t}\subseteq\pi^s_{t-1}\subseteq \dots \subseteq \pi^s_1.$$
That is, we observe $\{(X_t)_{H_{k,t'-1}}\}_{t=1}^{t'}$.
\end{proof}
Note that at any time-step $t$, $\SurveyUCB$ observes $(X_t)_{\pi^s_t}$. Lemma~\ref{lem:observed-covariates} shows that for any $t\in[t']$, we also observe $(X_t)_{H_{k,t'-1}}$ because the set of features queried by $\SurveyUCB$ at $t$ is a supper set of the support of the confidence set $\cset{k}{t'-1}$. That is, under the conditions of \cref{lem:observed-covariates} we have: $$\Supp(\cset{k}{t'-1})=H_{k,t'-1}\subseteq \pi^s_t, \text{ for all } t\in[t']. $$ 
\textbf{Initializing confidence sets:} We now define our confidence set for time-step zero based on \cref{ass:bounded-rewards}. For any arm~$k\in[K]$ under \cref{ass:bounded-rewards}, we have that $\|\beta_k\|_{\alpha I} = \sqrt{\beta_k(\alpha I)\beta_k} \leq \sqrt{\alpha}b$. That is, we have that $\beta_k\in\cset{k}{0}$ for all arms $k\in[K]$, where:
\begin{align*}
    \cset{k}{0}&:=\{\beta| \text{ } \|\beta \|_{\alpha I}\leq \sqrt{\alpha} \cdot b\}\\
    &=\{\beta| \text{ } \|\beta \|_{\alpha I} \leq \sqrt{\alpha} \cdot b, \Supp(\beta)\subseteq [d] \}.
\end{align*}
Based on \cref{ass:beta-min}, \cref{subsec:algconfidence} describes the construction of confidence set $\cset{k}{t'}$ for all arms $k\in[K]$ and time-steps $t'\geq 1$.  

\subsection{$\AlgConfidence$}
\label{subsec:algconfidence}
Consider any time-step $t'\geq 1$. Suppose that the confidence sets constructed upto time $t'$ are in Standard form, i.e. $\{\cset{i}{t}\}_{(i,t)\in[K]\times [t'-1]}$ are in Standard form. Let $k\in[K]$ be the arm pulled at time $t'$. We now describe the $\AlgConfidence$ update for confidence set $\cset{k}{t'}$ and show that the confidence sets $\{\cset{i}{t}\}_{(i,t)\in[K]\times [t']}$ are in Standard form. This would inductively imply that $\AlgConfidence$ constructs confidence sets in Standard form since the base case trivially holds \footnote{Confidence sets constructed up to and including time-step zero are trivially in Standard form.}.

From \cref{lem:observed-covariates} we already know that if the confidence sets constructed upto time $t'$ are in Standard form, then the decision maker using $\SurveyUCB$ at least observes $(X_t)_{H_{k,t'-1}}$ at every time-step $t\in[t']$. Where, $H_{k,t'-1}=\Supp(\cset{k}{t'-1})$. Hence, $\AlgConfidence$ can use this to construct the confidence set ($\cset{k}{t'}$) for arm $k$ at time $t'$.

\begin{algorithm}
    \caption{$\AlgConfidence$ (Constructing confidence sets in Standard form)}
    \label{alg:standard-form}
    \textbf{Input :} Given $\{(X_t)_{H_{k,t'-1}}\}_{t=1}^{t'}$, $Y_k$, $S_{k,t'}$, and $\cset{k}{t'-1}$, for some time-step $t'$ and arm $k$.\\
    \textbf{Output :} $\cset{k}{t'}$.
    \begin{algorithmic}[1] %
    \STATE Initialize $H_{k,t'}\leftarrow H_{k,t'-1}$.
    \FOR{$q\in[d]$}
        \IF{$\bmin > \max_{\beta\in\cset{k}{t'-1}} \|\beta_{\{q\}} \| $ } 
            \STATE $H_{k,t'} \leftarrow H_{k,t'}\setminus\{q\}$. 
        \ENDIF 
    \ENDFOR
    \STATE Set $\big(\trueEst_{k,t'}\big)_{H_{k,t'}^{\complement}} \leftarrow \Vec{0}$.
    \COMMENT{Note that $H_{k,t'}\subseteq H_{k,t'-1}$, hence we have $\{(X_t)_{H_{k,t'}}\}_{t=1}^{t'}$.}
    \STATE Estimate $(\trueEst_{k,t'})_{H_{k,t'}}$ from regression on the data $\Big\{\Big((X_t)_{H_{k,t'}}, Y_k(t) \Big)\Big\}_{t\in S_{k,t'}}$.
    \STATE Compute\\ $\Dtmatrix{k}{t'}:= (\alpha I)_{H_{k,t'}} + \sum_{w\in S_{k,t'}} (X_w)_{H_{k,t'}} (X_w)_{H_{k,t'}}^{\transpose}$.
    \STATE Output\\ $\cset{k}{t'}:=\{\beta| \text{ } \|\beta - \trueEst_{k,t'} \|_{\Dtmatrix{k}{t'}} \leq \conftBound{k}{t'}, \Supp(\beta)\subseteq H_{k,t'} \}$.
    \end{algorithmic}
\end{algorithm}

$\AlgConfidence$ starts by constructing $H_{k,t'}$ from $\cset{k}{t'-1}$ by relying on \cref{ass:beta-min}. Where $H_{k,t'}$ with be the support of the confidence set $\cset{k}{t'}$. We would like to construct $H_{k,t'}$ so that it contains the support of $\beta_k$, i.e. $\Supp(\beta_k)\subseteq H_{k,t'}$. In particular, if the confidence set for arm~$k$ at time $t'-1$ holds (i.e. $\beta_k\in\cset{k}{t'-1}$), then from \cref{ass:beta-min} we have that for all $q\in\Supp(\beta_k)$:
$$ \max_{\beta\in\cset{k}{t'-1}} \|\beta_{\{q\}} \| \geq \|(\beta_k)_{\{q\}}\| > \bmin. $$
Also we get that $\Supp(\beta_k)$ is a subset of the support of the confidence set ($H_{k,t'-1}$) of arm $k$ at time $t'-1$. Hence we have that:
$$ H_{k,t'} := H_{k,t'-1} \Big\backslash \bigg\{q| \max_{\beta\in\cset{k}{t'-1}} \|\beta_{\{q\}} \| \leq \bmin\bigg\} $$

$\AlgConfidence$ now constructs confidence set $\cset{k}{t'}$ with support $H_{k,t'}$. We then estimate $(\trueEst_{k,t'})_{H_{k,t'}}$ by regressing over the features in $H_{k,t'}$, on the observed data set: $\Big\{\Big((X_t)_{H_{k,t'}}, Y_k(t) \Big)\Big\}_{t\in S_{k,t'}}$. We then set the components of $\trueEst_{k,t'}$ not in $H_{k,t'}$ to zero, i.e. $\big(\trueEst_{k,t'}\big)_{H_{k,t'}^{\complement}} \leftarrow \Vec{0}$. Now with $\Dtmatrix{k}{t'}:= (\alpha I)_{H_{k,t'}} + \sum_{w\in S_{k,t'}} (X_w)_{H_{k,t'}} (X_w)_{H_{k,t'}}^{\transpose}$, we construct the confidence set for arm~$k$ at time $t'$:
\begin{align*}
    \cset{k}{t'}:=\{\beta| \text{ } \|\beta - \trueEst_{k,t'} \|_{\Dtmatrix{k}{t'}} \leq \conftBound{k}{t'}, \Supp(\beta)\subseteq H_{k,t'} \}%
\end{align*}

Note that $\Supp(\cset{k}{t'})=H_{k,t'}\subseteq H_{k,t'-1}$ and $\Supp(\trueEst_{k,t'}) \subseteq H_{k,t'}$. Hence given the above form of the confidence set, we have that $\{\cset{i}{t}\}_{(i,t)\in[K]\times [t']}$ are in Standard form.

\subsection{Probability Aggregation}
\label{subsec:prob-agg-outline}

To construct the confidence set $\cset{k}{t'}$, recall that $\AlgConfidence$ assumes that $\beta_k\in\cset{k}{t'-1}$. This is unlike LinUCB \cite{li2010contextual} and several other UCB algorithms where confidence sets are constructed from observed data without directly relying on previous confidence sets. Here, we argue that our construction does not lead to any unexpected issues. We now state a helpful lemma and its corollary, and defer proofs to \cref{app:probability-aggregation-proof}.

\begin{restatable}[Probability aggregation]{lemma}{probAgg}
\label{lem:probability-aggregation}
Consider a probability space $(\Omega,\F,\Pr)$. Consider any sequence of events $\{B_i,\Pi_i\}_{i=1}^{\infty}$, such that $B_i,\Pi_i\in\F$ and $B_i \subseteq \Pi_i$ for any $i\in\mathcal{N}$. Let $\Pi_0:=\Omega$. We then have that:
$$ \Pr\Bigg[\bigcap_{i=1}^{\infty} B_i \Bigg] \geq 1 - \sum_{i=1}^{\infty}\Pr\big[B_i^{\complement}|\Pi_{i-1}\big]. $$
\end{restatable}

\begin{restatable}{corollary}{corAgg}
\label{cor:prob-aggregation}
Suppose $\SurveyUCB$ constructs the Standard form confidence sets $\{\cset{i}{t-1}\}_{(i,t)\in[K]\times[T]}$ for all arms $k$ and all time-step's $t$. We then have that:
\begin{align*}
    &\Pr\Bigg[\bigcap_{k=1}^K\bigcap_{t=1}^{\infty} \{\beta_k\in\cset{k}{t-1} \} \Bigg] \\ 
    &\geq 1 - \sum_{k=1}^K\sum_{t\in S_k}\Pr[\beta_{k}\notin\cset{k}{t}|\Supp(\beta_{k})\subseteq H_{k,t}].
\end{align*}
Where $H_{k,t-1} := \Supp(\cset{k}{t-1})$ and $S_k:=\{t| \pi_t^a=k\}$ for all arms $k\in[K]$ and time-steps $t\in[t']$.
\end{restatable}

Therefore from \cref{cor:prob-aggregation} for any $\delta>0$, we have that $\Pr[\bigcap_{k=1}^K\bigcap_{t=1}^{\infty} \{\beta_k\in\cset{k}{t-1} \} ]\geq 1 - \delta$ if for all $k\in[K]$ and $t\in[T]$ we have: \footnote{Note that $\sum_{j=1}^{\infty}\frac{1}{(1+j)^2} < \frac{3.15^2}{6} - 1 < 1.$}
\begin{align}
\label{eq:arm-prob-bound}
    \Pr[\beta_{k}\notin\cset{k}{t}|\Supp(\beta_{k})\subseteq H_{k,t}] \leq \frac{\delta}{K (1+n_{k,t})^2}
\end{align}

\section{General Regret Analysis}
\label{sec:general_regret_analysis}

In \cref{subsec:algconfidence}, we already saw that $\SurveyUCB$ constructs confidence sets in Standard form. In \cref{lem:general-regret-analysis} we exploit the structure of confidence sets in Standard form to get a general regret bound for $\SurveyUCB$. 
\begin{restatable}[General regret analysis]{lemma}{genregbound}
\label{lem:general-regret-analysis}
Suppose \cref{ass:bounded-rewards} holds with $\|X_t\|_2\leq L$ for all $t$. Suppose the confidence sets $\{\cset{i}{t-1}\}_{(i,t)\in[K]\times[T]}$ used in $\SurveyUCB$ are in Standard form with parameter $\alpha > 0$ (in \cref{def:standard-form-confidence}). And suppose we have that $\beta_i\in\cset{i}{t-1}$ at any time-step $t$ and for all arms $i\in[K]$. We then have that:
\begin{itemize}
    \item Instantaneous regret at time-step $t$, $r_t\leq 2\theta_{\pi^a_t,t-1}\|(X_t)_{\pi^s_t} \|_{\VtmatrixInv}$.
    \item And cumulative regret in $T$ rounds, $R_T \leq \sum_{i\in[K]} \theta_{i,T-1}\sqrt{8n_{i,T}d\log\Big(\frac{d\alpha+n_{i,T}L^2}{d\alpha} \Big)}.$
\end{itemize}
\end{restatable}
We defer the proof of \cref{lem:general-regret-analysis} \cref{app:general-regret-analysis}. It is worth noting that the proof sheds some light on the need for confidence sets to be in Standard form.

\section{Tail Inequalities for Adapted Observations}

Consider a linear model $Y=\X\beta + \epsilon$, with design matrix $\X\in\R^{n\times d}$, response vector $Y\in\R^n$, and noise vector $\epsilon\in\R^n$. Where $\epsilon_{t}$ are independent sequence of $\sigma$-sub-Gaussian random variables.

\subsection{Ridge Estimator}

We now define the Ridge estimator for estimating the parameter $\beta$ as follows:

\begin{definition}[Ridge]
Given regularization parameters $\alpha \geq 0$. The Ridge estimate is given by:
$$ \trueEst^{\ridge}_{\X,Y}(\alpha):=\arg\min_{\beta'} \bigg\{\|Y-\X\beta'\|_2^2 + \alpha\|\beta'\|_2^2 \bigg\} $$
\end{definition}

From theorem 2 in \cite{abbasi2011improved} we get:

\begin{lemma}[Abbasi-Yadkori, et al 2] \label{lem:abassi-confidence-set}
Let $X_t$ denote the $t$-th row of $\X$. Let $Y(t)$ denote the $t$-th entry of $Y$. The sequence $\{X_t| t=1,2,\dots,n \}$ form an adapted sequence of observations. That is, $X_t$ may depend on $\{X_{t'}, Y(t') \}_{t'=1}^{t-1}$. Assume that $\|\beta \|_2 \leq b$. Then for $\delta>0$, we have that: 
$$ \Pr \bigg[\|\trueEst - \beta \|_D^2 \leq \sigma \sqrt{2\log\bigg(\frac{\det^{1/2}(D)}{\delta\alpha^{d/2}}\bigg)} + \sqrt{\alpha}b  \bigg] \geq 1-\delta $$
Furthermore, if $\|X_t\|_2\leq L$ for all $t\in[n]$, we have that:
$$ \Pr \bigg[\|\trueEst - \beta \|_D^2 \leq \sigma \sqrt{d\log\bigg(\frac{1+nL^2/\alpha}{\delta}\bigg)} + \sqrt{\alpha}b  \bigg] \geq 1-\delta $$
Where $\trueEst$ is the Ridge estimate with parameter $\alpha > 0$. And, $D=\X^{\transpose}\X + \alpha I$.
\end{lemma}

\subsection{Elastic Net Estimator}

We now define the Elastic net estimator for estimating the parameter $\beta$ as follows:

\begin{definition}[Elastic net]
Given regularization parameters $\alpha \geq 0$, $\lambda \geq 0$. The Elastic net estimate is given by:
\begin{align*}
    &\trueEst^{\elnet}_{\X,Y}(\alpha, \lambda)\\
    &:=\arg\min_{\beta'} \bigg\{\frac{1}{n}\Big[\|Y-\X\beta'\|_2^2 + 2\alpha\|\beta'\|_2^2 \Big] + \lambda\|\beta'\|_1 \bigg\} 
\end{align*}
\end{definition}

We now provide an adaptive tail inequality for the Elastic net estimator that may be of independent interest. We defer the proof to \cref{app:elnet-tail-inequality}.

\begin{restatable}[Elastic net tail inequality for Adapted observations]{lemma}{elnetTail}\label{lem:adapted-elnet-tail}
Let $X_t$ denote the $t$-th row of $\X$. Let $Y(t)$ denote the $t$-th entry of $Y$. The sequence $\{X_t| t=1,2,\dots,n \}$ form an adapted sequence of observations. That is, $X_t$ may depend on $\{X_{t'}, Y(t') \}_{t'=1}^{t-1}$. Also assume all realizations of $X_t$ satisfy $\|X_t\|_{\infty} \leq L$ and that $\|\beta \|_1 \leq b$. Then, if $\lambda := 4\sigma L \sqrt{(\gamma^2+2\log d)/n}$, we have:
\begin{align*}
    \Pr \bigg[\|\trueEst - \beta \|_D^2 \leq 6\sigma L b \sqrt{n(\gamma^2+2\log(d))} + 4\alpha b^2 \bigg]\\ 
\geq 1-2\exp[-\gamma^2/2].
\end{align*}
Where $\trueEst$ is the Elastic net estimate with parameters $\alpha,\lambda \geq 0$. And, $D=\X^{\transpose}\X + \alpha I$.
\end{restatable}

\section{Ridge $\SurveyUCB$}
\label{sec:ridge-surveyucb}

Note that the description and analysis of $\SurveyUCB$ gives us a fair amount of flexibility. We are still free to specify the regression method used to estimate $(\trueEst_{i,t})_{H_{i,t}}$ for all $i\in[K]$ and times $t\in[T]$. We are also free to specify our choice of $\conftBound{i}{t}$ for all arms $i\in[K]$ and all times $t\in[T]$.

Ridge $\SurveyUCB$ is a version of the $\SurveyUCB$. In Ridge $\SurveyUCB$, we use Ridge regression with a fixed regularization parameter ($\alpha>0$), to estimate $(\trueEst_{i,t})_{H_{i,t}}$. We also choose $\conftBound{i}{t}$ for all arms and time-steps based on \cref{cor:prob-aggregation} (\cref{eq:arm-prob-bound}) and \cref{lem:abassi-confidence-set} so that the Standard form confidence sets that we construct hold with high probability. Then from \cref{lem:general-regret-analysis} we naturally get a high-probability regret bound for Ridge $\SurveyUCB$.

Say we want our bounds to hold with probability $1-\delta$. Then from \cref{cor:prob-aggregation}, it is enough to show that \cref{eq:arm-prob-bound} holds for all arms and times. That is for all arms $k\in[K]$ and all times $t$, we want:
\begin{align*}
    \Pr[\beta_{k}\notin\cset{k}{t}|\Supp(\beta_{k})\subseteq H_{k,t}] \leq \frac{\delta}{K (1+n_{k,t})^2} \tag{\ref{eq:arm-prob-bound}}
\end{align*}
Given $\Supp(\beta_{k})\subseteq H_{k,t}$, \cref{lem:abassi-confidence-set} gives us that \cref{eq:arm-prob-bound} holds for $\cset{k}{t}$ with support $H_{k,t}$ if: %
\begin{align}
\label{eq:ridge-conftbound}
    \conftBound{k}{t} = \sigma \sqrt{|H_{k,t}|\log\bigg(\frac{1+n_{k,t}L^2/\alpha}{\delta/(K (1+n_{k,t})^2)}\bigg)} + \sqrt{\alpha}b
\end{align}
Hence, from \cref{lem:general-regret-analysis} with the above choice of $\conftBound{k}{t}$ and the corresponding high-probability guarantee, we get \cref{thm:ridge-survey-ucb-regret}. We defer the proof to \cref{app:ridge-regret}.
\begin{restatable}[Ridge $\SurveyUCB$ regret]{theorem}{ridgeRegret}
\label{thm:ridge-survey-ucb-regret}
Suppose \cref{ass:beta-min} and \cref{ass:bounded-rewards} hold. Let $\delta,\alpha > 0$ be fixed constants. Let $\conftBound{k}{t}$ be chosen as in \cref{eq:ridge-conftbound} for all $k,t$. Then with probability at least $1-\delta$, Ridge $\SurveyUCB$ has the following regret guarantee:
\begin{align*}
    R_T \leq O(d \sqrt{KT\log(K)} \log(T)).
\end{align*}
\end{restatable}

\section{Elastic Net $\SurveyUCB$}
\label{sec:elastic-net}

In this section, we want to develop a variant of $\SurveyUCB$ that is more robust to the choice of the beta-min parameter in \cref{ass:beta-min}. One way to do this is to modify line 3 %
in $\AlgConfidence$. In particular, suppose $k\in[K]$ is the arm pulled at time $t'$. We then construct $H_{k,t'}$ by removing all features $q$ from $H_{k,t'-1}$ that we estimate to be zero (i.e. $(\trueEst_{k,t'-1})_{q}=0$) and that we determine are irrelevant based on \cref{ass:beta-min}. That is:
\begin{align*}
    H_{k,t'} := H_{k,t'-1} \Big\backslash \bigg\{q \suchthat & (\trueEst_{k,t'-1})_{q}=0 \\
    &\text{ and }\max_{\beta\in\cset{k}{t'-1}} \|\beta_{\{q\}} \| \leq \bmin \bigg\} 
\end{align*}
It is easy to see that all our arguments for $\SurveyUCB$ continue to hold with the above modification. Now note that the above modification makes $\SurveyUCB$ more robust to \cref{ass:beta-min}, because we additionally need the $q$th coordinate of our estimate ($\trueEst_{k,t'-1}$) to be zero before we remove feature $q$ (at time $t'$) from the model of arm~$k$. This modification encourages us to use sparse estimators.

Elastic Net $\SurveyUCB$ is a variant of $\SurveyUCB$ with the above modification. In Elastic Net $\SurveyUCB$, we use the Elastic net regression to estimate $\trueEst_{i,t}$, with regularization parameters $\alpha$ and $\lambda_{i,t}$ for all arms $i\in[K]$ and time-steps $t$. Where:
\begin{align}
\label{eq:elnet-L1reg}
    \lambda_{i,t} := 4\sigma L \sqrt{\frac{2}{n_{i,t}} \log\bigg(\frac{4dKn_{i,t}^2}{\delta} \bigg)},\text{ } \forall i\in[K], \forall t\in[T]. 
\end{align}
Similar to the arguments in \cref{sec:ridge-surveyucb}, \cref{cor:prob-aggregation} and \cref{lem:adapted-elnet-tail} give us that our confidence sets hold with probability $1-\delta$ if for all arms $i$ and times $t$:
\begin{align}
\label{eq:elnet-conftbound}
    \conftBound{i}{t} := \sqrt{6\sigma L b\sqrt{2n_{i,t}\log\bigg(\frac{4dKn_{i,t}^2}{\delta}\bigg)} + 4\alpha b^2}.
\end{align}
Again from \cref{lem:general-regret-analysis} with the above choice of $\conftBound{k}{t}$ and the corresponding high-probability guarantee, we get \cref{thm:elastic-net-survey-ucb-regret}. We defer the proof to \cref{app:elnet-regret}.
\begin{restatable}[Elastic Net $\SurveyUCB$ regret]{theorem}{elnetRegret}
\label{thm:elastic-net-survey-ucb-regret}
Suppose \cref{ass:beta-min} and \cref{ass:bounded-rewards} hold. Let $\delta,\alpha > 0$ be fixed constants. For all $k,t$, let $\lambda_{k,t}$ and  $\conftBound{k}{t}$ be chosen as in \cref{eq:elnet-L1reg} and \cref{eq:ridge-conftbound} respectively. Then with probability at least $1-\delta$, Elastic Net $\SurveyUCB$ has the following regret guarantee:
\begin{align*}
    R_T \leq O(K^{1/4} d^{1/2} T^{3/4}  \log^{3/4}(T) \log^{1/4}(dK) ).
\end{align*}
\end{restatable}

\section{Interactive Surveys}
We start by defining sub-optimal arms, describe an inefficiency in $\SurveyUCB$, and propose a fix using interactive surveys.
\begin{definition}[Sub-optimal arms] 
An arm is said to be sub-optimal if the target policy would not pick it for any context vector in the context space.
\end{definition}
At any time-step $t$, the $\SurveyUCB$ algorithms query $\pi^s_t := \cup_{i\in[K]} \pi^s_{i,t}, \text{ with } \pi^s_{i,t}=\Supp(\cset{i}{t-1}) \text{ for all } i\in[K].$ Now suppose the decision maker plays arm~$k$ at time $t$. The decision maker only needs the reward and the features corresponding to $\pi^s_{k,t}$ to update the model. Recall that the reason the decision maker queries all the features in $\pi^s_t$ is to be able to compute the upper confidence bounds for all arms. 

Note that the bandit assignment only depends on the highest upper confidence bound, so it is not necessary to compute all the upper confidence bounds. The algorithm's inefficiency becomes more evident in the presence of sub-optimal arms because these arms are played less frequently, hence updated less frequently, and increase survey length.

\begin{algorithm}
    \caption{Interactive Survey Protocol}
    \label{alg:interactive-survey}
    \textbf{Interactive input :} User at time $t$, that answers queries whenever asked.\\
    \textbf{Output:} Queries to user $t$. 
    \begin{algorithmic}[1] %
    \STATE Create an ordered list of arms $Q$, starting from the most pulled arm to least pulled arm.
    \STATE Initialize $M\leftarrow -\infty$, largest upper confidence bound among queried arms.
    \STATE Initialize set of queried features $U \leftarrow \emptyset$.
    \WHILE{$Q$ is not empty}
        \STATE Let $i \leftarrow Q[1]$, be the first arm in $Q$.
        \STATE Query $\pi^s_{i,t}$ and observe $(X_t)_{\pi^s_{i,t}}$.
        \STATE Update $U\leftarrow U \cup \pi^s_{i,t}$.
        \STATE Update $M \leftarrow \max\{M, \max_{\beta \in \cset{k}{t-1}} (X_t)_{\pi^s_t}^{\transpose}\beta \}$
        \STATE Remove $i$ from the list $Q$.
        \FOR{$w$ in $Q$}
            \STATE Set $F=\{x|x_U = (X_t)_U \text{ and $x$ is feasible}  \}$.
            \IF{$\max_{x\in F}\max_{\beta \in \cset{w}{t-1}} x^{\transpose}\beta \leq M$}
                \STATE Remove arm $w$ from $Q$.
            \ENDIF
        \ENDFOR
    \ENDWHILE\\
    \COMMENT{Note that we observe enough information to determine which arm has the largest upper confidence and can update its confidence set after observing its reward.}
    \end{algorithmic}
\end{algorithm}

We attempt to resolve this inefficiency through interactive surveys. Consider the user at time $t$, the decision maker needs to query some features from the user before taking an action. We start by creating an ordered list of all arms, starting from the most pulled arms and ending with the least pulled arms. This ordering is a heuristic choice, the idea is to keep sub-optimal arms (which are less frequently pulled) towards the end of the list. We keep removing arms from this list and terminate (take an action) when it is empty. 

The main idea is to sequentially query the feature sets $\pi^s_{i,t}$ for arms $i$ in the list, simultaneously remove queried arms from the list, and also remove unqueried arms that do not have the largest upper confidence bound. We will now explain how we determine that some unqueried arm $w$ doesn't have the largest upper confidence bound. First note that for each queried arm, we can exactly compute its upper confidence bound at time $t$. Clearly the following optimization problem gives us an upper bound to the upper confidence bound of arm~$w$: \footnote{Here $(\Dtmatrix{w}{t-1})^{-1}$ is the pseudo-inverse of the matrix.}
\begin{equation}
\label{eq:optimization-interactive}
\begin{aligned}
\max_{x\in F} \quad &  \max_{\beta \in \cset{w}{t-1}} x^{\transpose}\beta\\
\equiv \max_{x\in F} \quad &  x^{\transpose}\trueEst_{w,t-1}+\sqrt{\conftBound{w}{t-1}(x^{\transpose}(\Dtmatrix{w}{t-1})^{-1}x)}%
\end{aligned}
\end{equation}
Where $F$ denotes the set of contexts in the context space that are consistent with the queries so far (i.e. feasible $x$ such that $x_U = (X_t)_U$ where $U$ is the set of features queried so far for user $t$). Hence if this upper bound is less than the upper confidence bound for some queried arm, we remove arm~$w$ from the list. For more details see \cref{alg:interactive-survey}.

Note that interactive versions of $\SurveyUCB$ have the same regret performance as $\SurveyUCB$, because both algorithms work with the same confidence sets and always choose the arm with the largest upper confidence bound. Also note that the optimization problem in \cref{eq:optimization-interactive} is non-convex. In \cref{app:interactive-survey} we provide a heuristics for this optimization that has been effective in simulations.

\section{Simulation}

Suppose we have users with fifty features and suppose we have five arms. Where expected reward for context $x$ is: $x_1$ for arm~1, $x_2$ for arm~2, $1-x_1$ for arm~3, and zero for both arm~4 and arm~5. Hence, only two of the fifty features are predictive of arm rewards and both arm~4 and arm~5 are sub-optimal. We draw contexts from the uniform distribution $U([0,1]^{d})$ and reward noise ($\epsilon_{i,t}$) is drawn from $U([0,1])$. We consider a 100000 step time horizon.

We assume noise is 1-sub-Gaussian, contexts lie in the space $[0,1]^{d}$, and the 1-norm and 2-norms of the arm parameters is bounded by $50$ and $\sqrt{50}$ respectively. We run simulations with ($K=5$) and without ($K=3$) sub-optimal arms. That is, in our simulations without sub-optimal arms, we only consider the first three arms. We plot regret and cumulative survey length vs time-steps. The plots here are generated by averaging over five runs.

In simulations, when \cref{ass:beta-min} holds, Ridge $\SurveyUCB$ has better regret performance compared to Elastic Net $\SurveyUCB$ (for similar beta-min parameters). Plots verify that Elastic Net $\SurveyUCB$ is infact reasonably robust to \cref{ass:beta-min} and doesn't remove predictive features (in simulations) even under violations of \cref{ass:beta-min}. We also note that sub-optimal arms hurt the performance of $\SurveyUCB$ algorithms. And, interactive surveys help mitigate the negative effects of sub-optimal arms on survey lengths. Also note that performance on regret improves with less conservative choices for the beta-min parameter, this implies that model truncation also helps improve regret performance.

\begin{figure}[!tbp]
  \centering
  \begin{minipage}[b]{0.23\textwidth}
    \includegraphics[width=\textwidth]{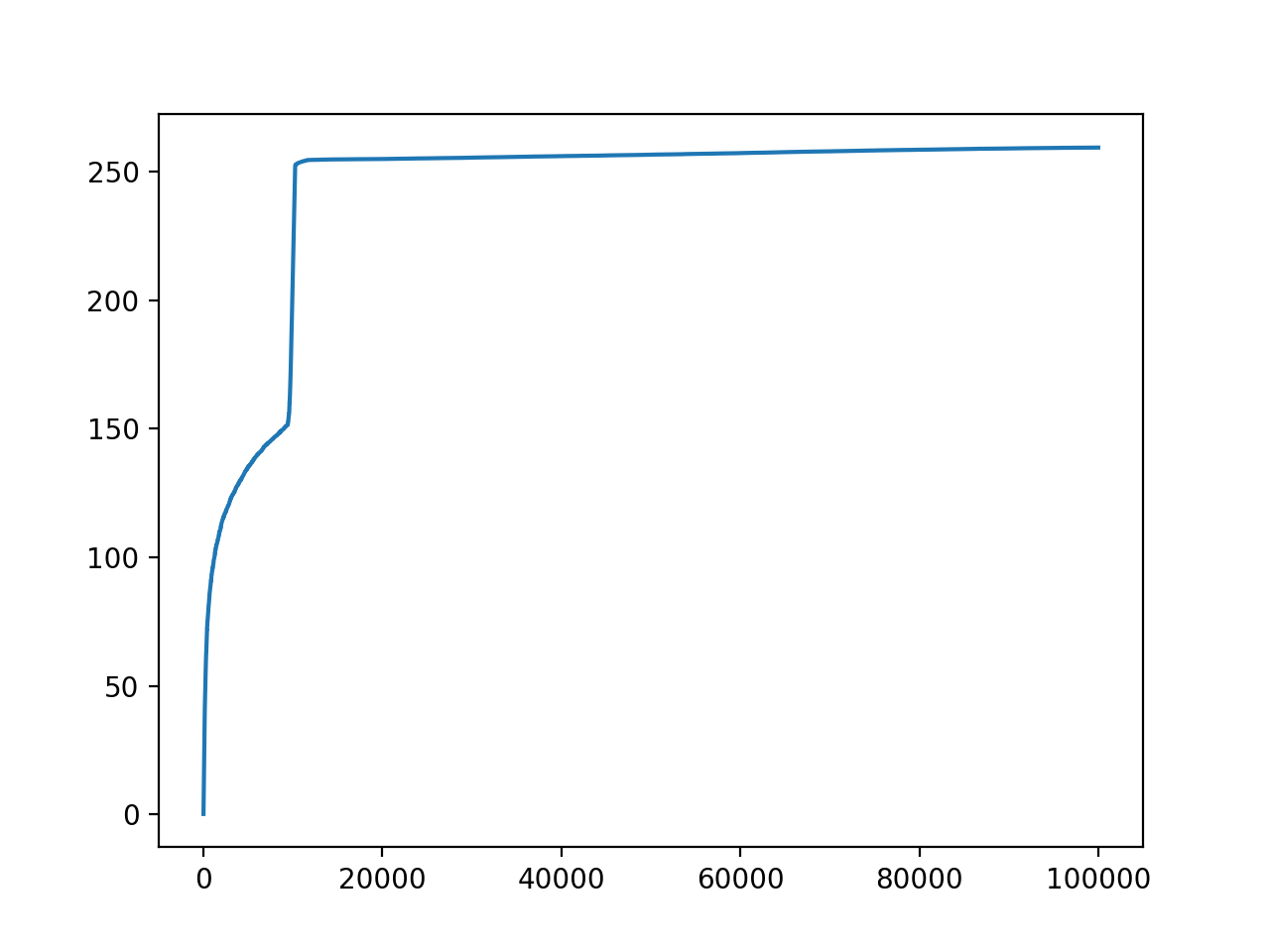}
    \caption{Regret.}
  \end{minipage}
  \hfill
  \begin{minipage}[b]{0.23\textwidth}
    \includegraphics[width=\textwidth]{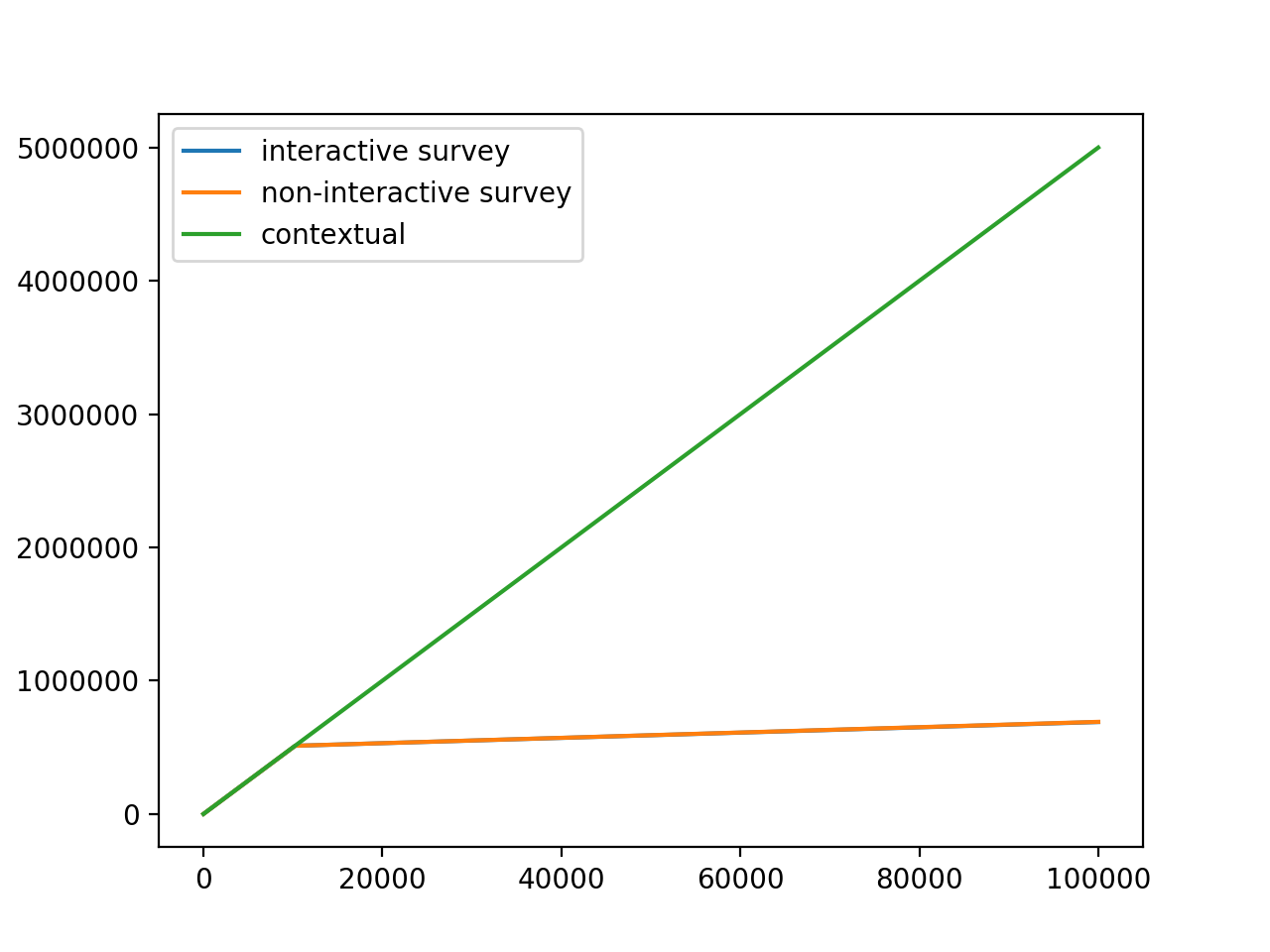}
    \caption{Survey length.}
  \end{minipage}
  Ridge $\SurveyUCB$ with $\bmin = 0.3, K=3$
\end{figure}

\begin{figure}[!tbp]
  \centering
  \begin{minipage}[b]{0.23\textwidth}
    \includegraphics[width=\textwidth]{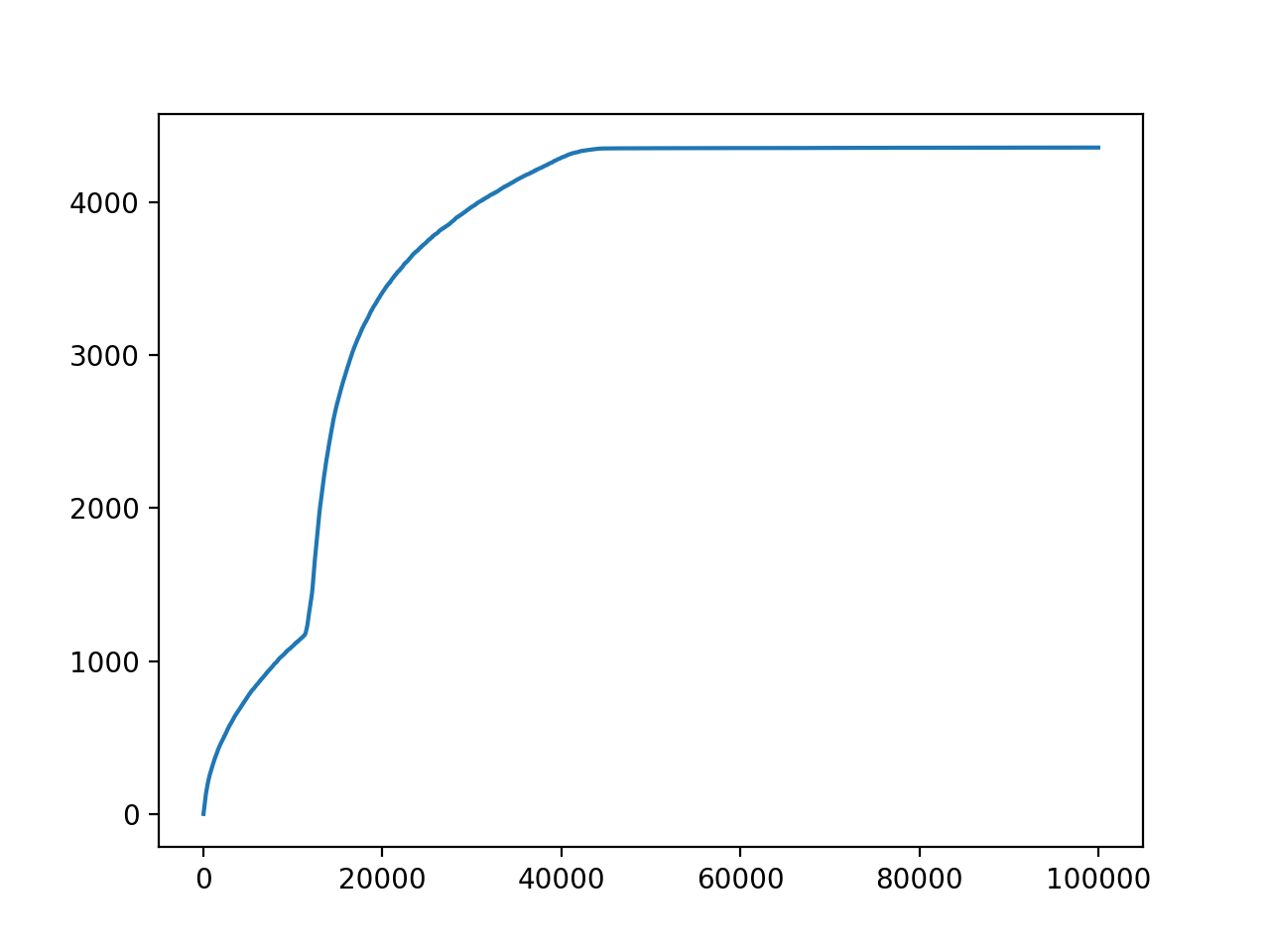}
    \caption{Regret.}
  \end{minipage}
  \hfill
  \begin{minipage}[b]{0.23\textwidth}
    \includegraphics[width=\textwidth]{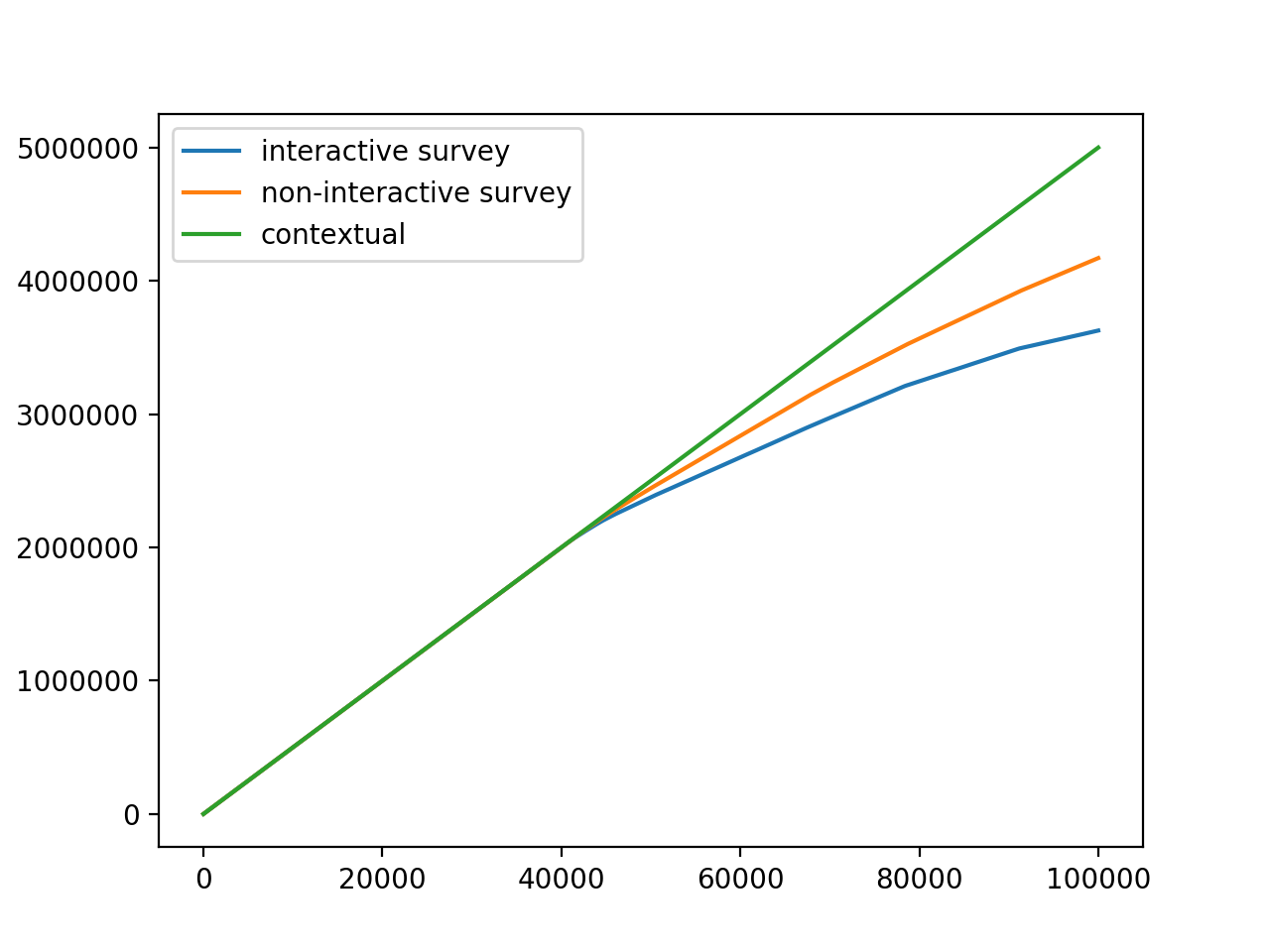}
    \caption{Survey length.}
  \end{minipage}
  Ridge $\SurveyUCB$ with $\bmin = 0.3, K=5$
\end{figure}

\begin{figure}[!tbp]
  \centering
  \begin{minipage}[b]{0.23\textwidth}
    \includegraphics[width=\textwidth]{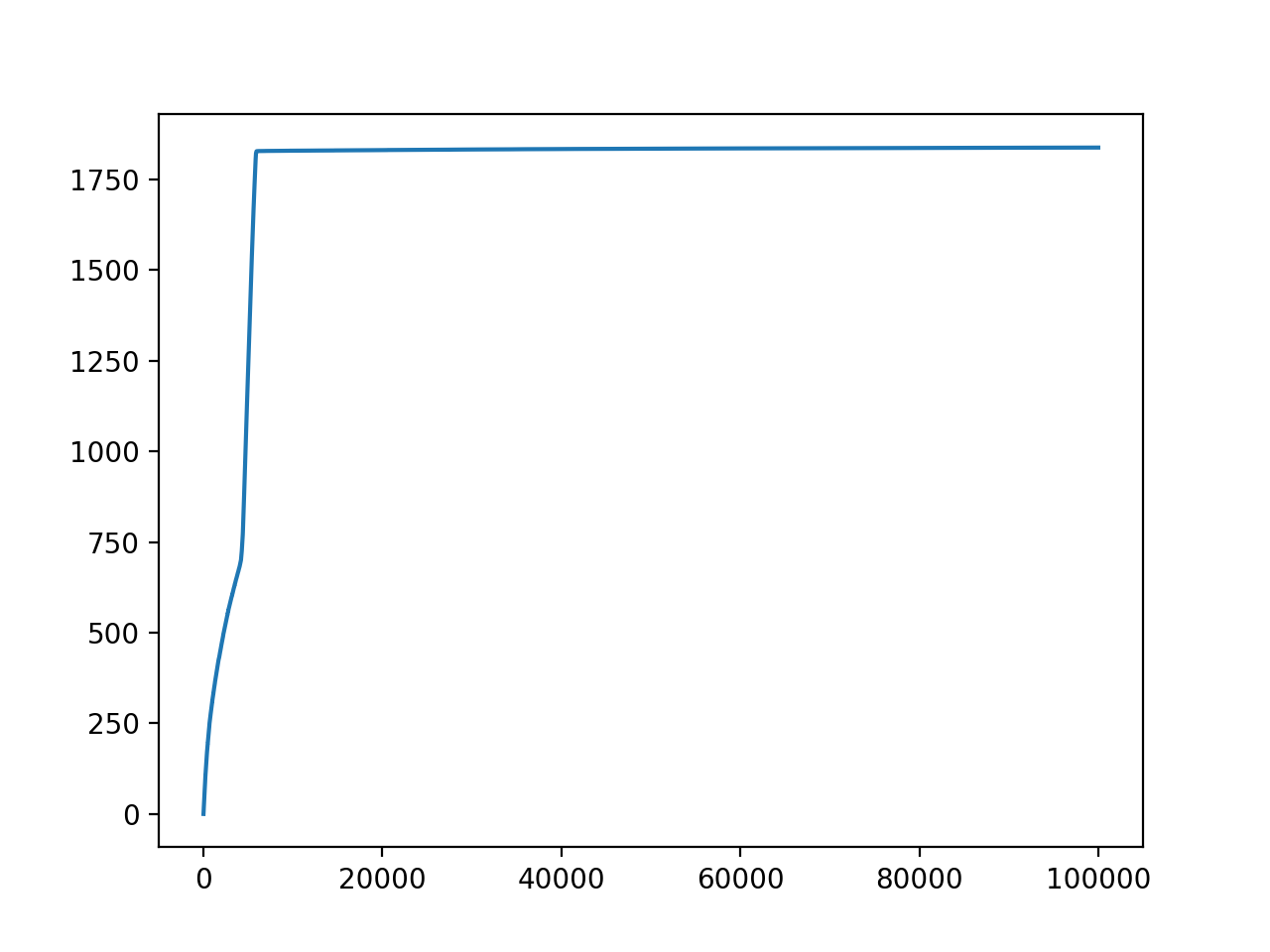}
    \caption{Regret.}
  \end{minipage}
  \hfill
  \begin{minipage}[b]{0.23\textwidth}
    \includegraphics[width=\textwidth]{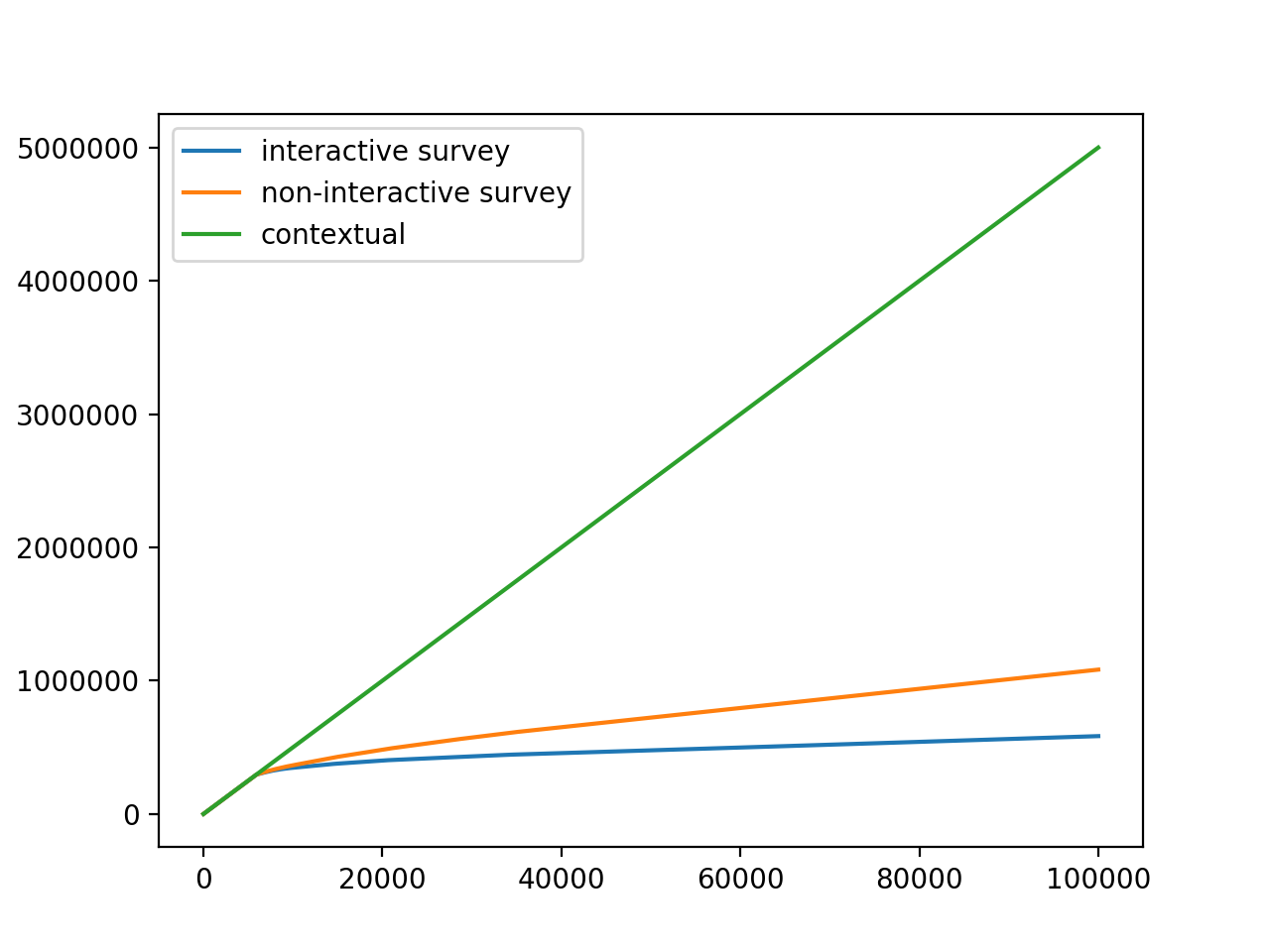}
    \caption{Survey length.}
  \end{minipage}
  Ridge $\SurveyUCB$ with $\bmin = 0.5, K=5$
\end{figure}

\begin{figure}[!tbp]
  \centering
  \begin{minipage}[b]{0.23\textwidth}
    \includegraphics[width=\textwidth]{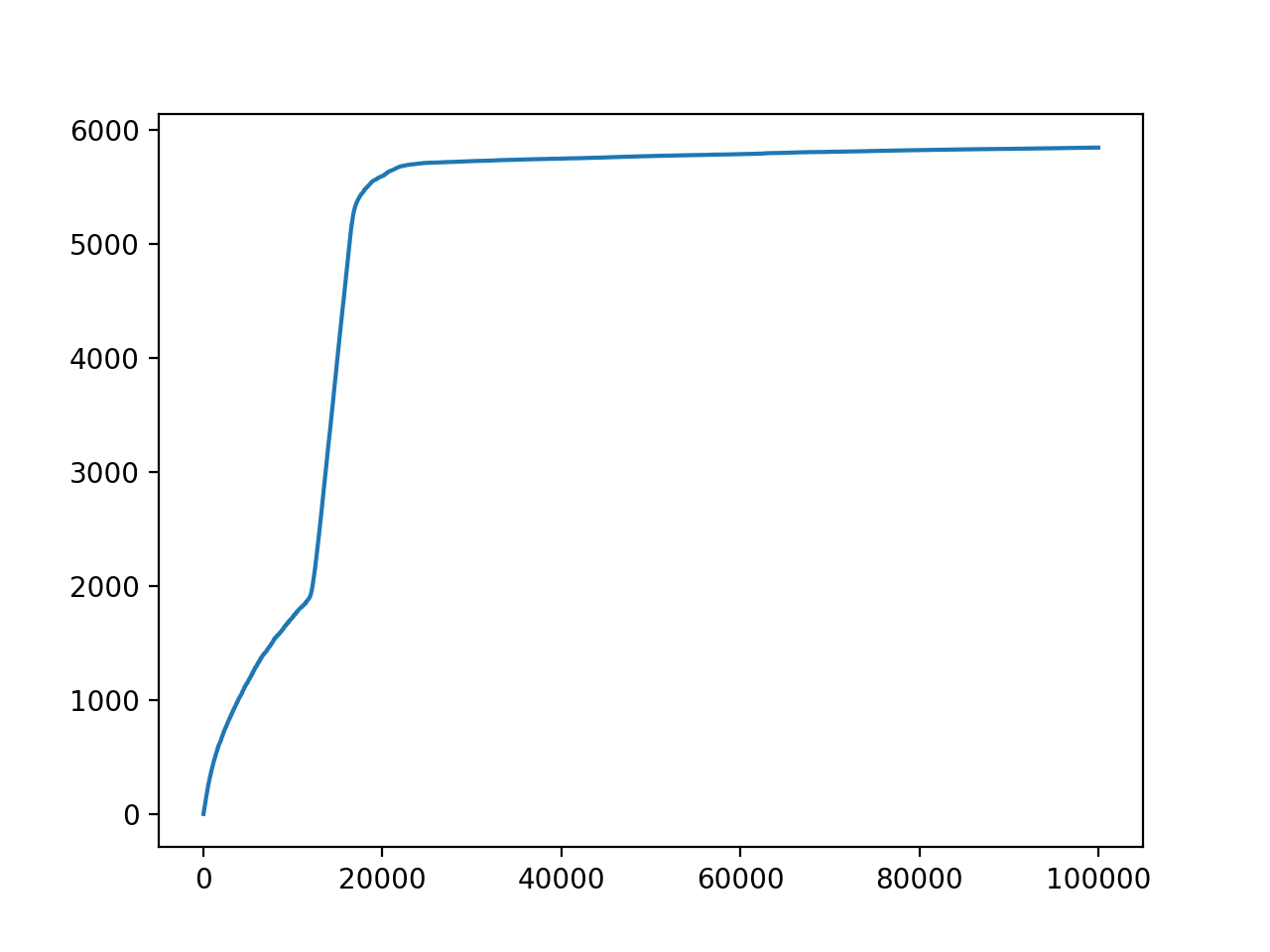}
    \caption{Regret.}
  \end{minipage}
  \hfill
  \begin{minipage}[b]{0.23\textwidth}
    \includegraphics[width=\textwidth]{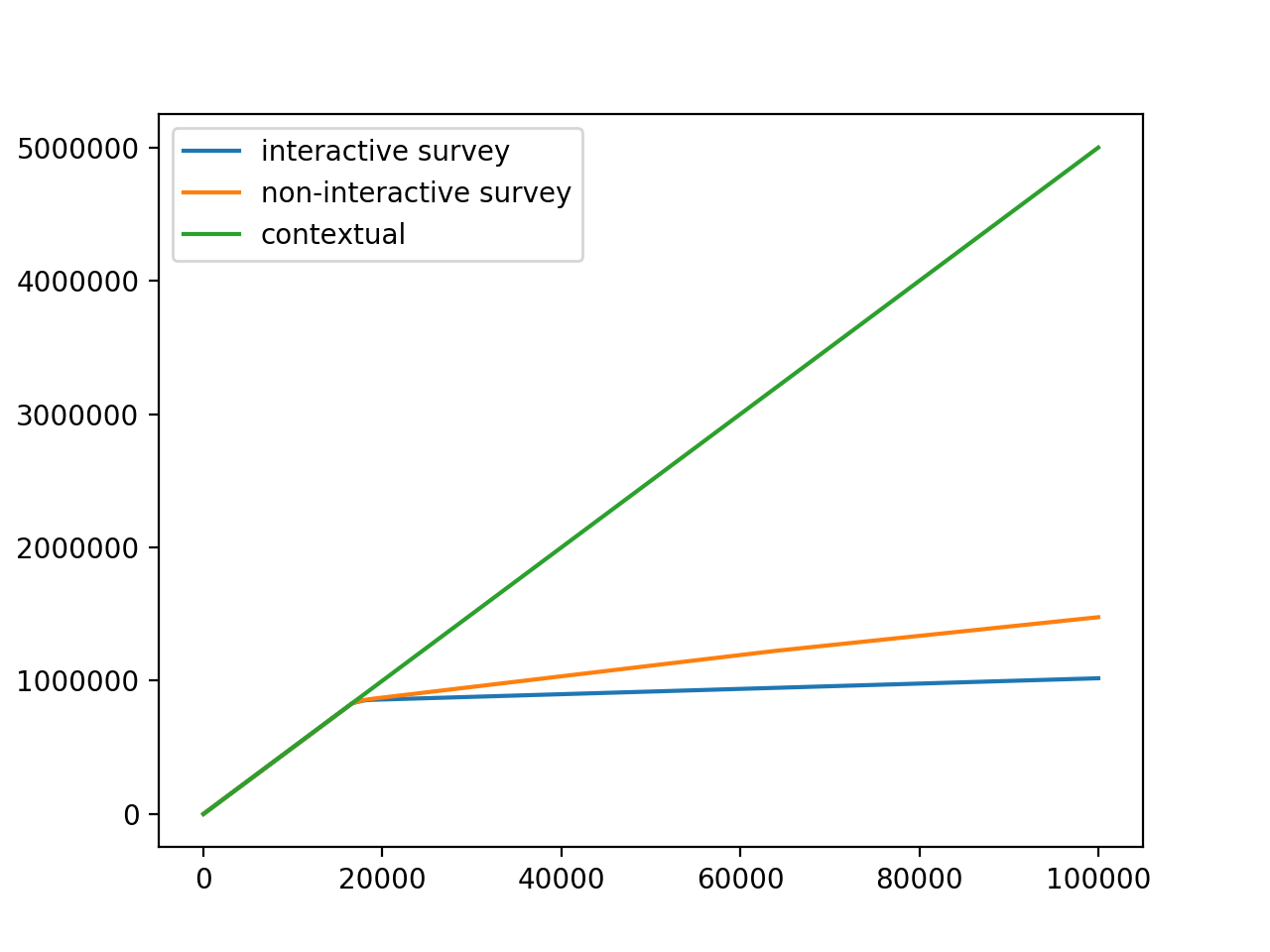}
    \caption{Survey length.}
  \end{minipage}
  Elastic net $\SurveyUCB$ with $\bmin = 0.7, K=5$.
\end{figure}

\begin{figure}[!tbp]
  \centering
  \begin{minipage}[b]{0.23\textwidth}
    \includegraphics[width=\textwidth]{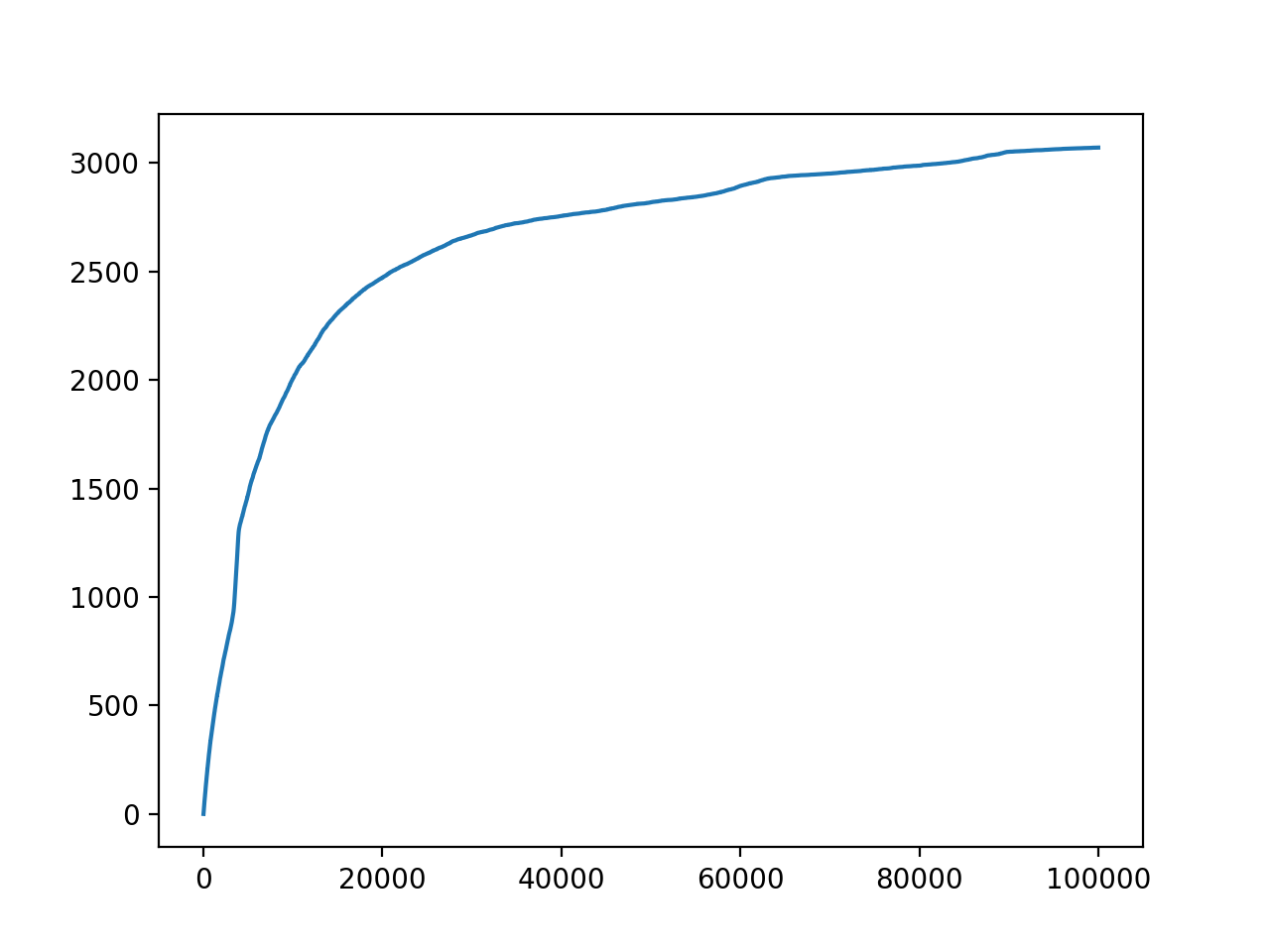}
    \caption{Regret.}
  \end{minipage}
  \hfill
  \begin{minipage}[b]{0.23\textwidth}
    \includegraphics[width=\textwidth]{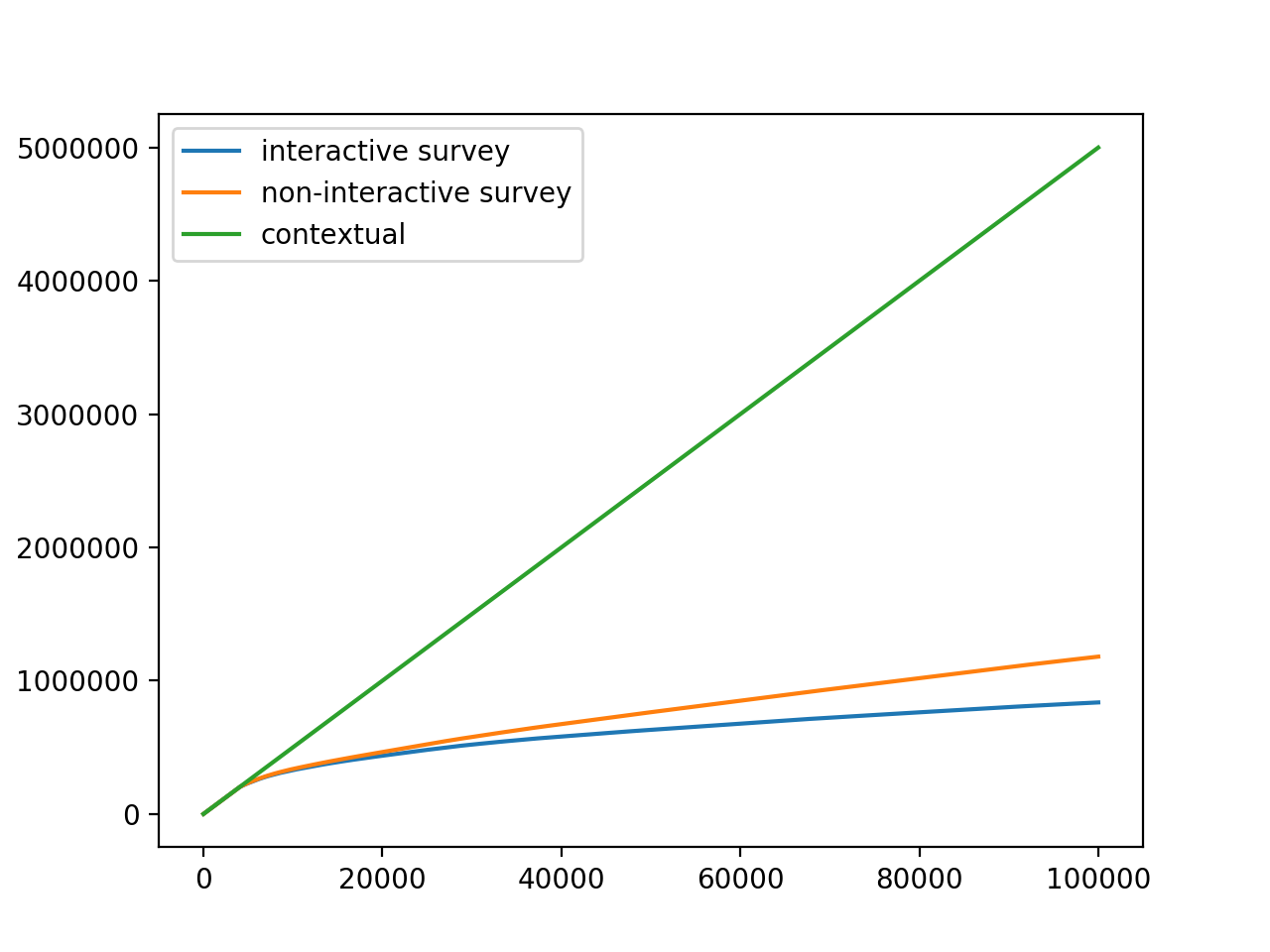}
    \caption{Survey length.}
  \end{minipage}
  Elastic net $\SurveyUCB$ with $\bmin = 1.5, K=5$\\ 
  Note \cref{ass:beta-min} is violated.
\end{figure}

\bibliography{ref}
\bibliographystyle{icml2020}

\appendix
\clearpage
\onecolumn

\section{Proofs for probability aggregation}
\label{app:probability-aggregation-proof}

\probAgg*
\begin{proof}
We want to use induction. First note that: 
$$\Pr[B_1] = 1 - \Pr[B_1^{\complement}] = 1 - \Pr[B_1^{\complement}|\Omega].$$
Also note that for any $t\geq 1$, we have that:
\begin{align*}
    \Pr[\cap_{i=1}^{t+1} B_i] & = \Pr[\cap_{i=1}^{t} B_i] - \Pr[\cap_{i=1}^{t} B_i \cap B_{t+1}^{\complement}] & \\
    & \geq \Pr[\cap_{i=1}^{t} B_i] - \Pr[\Pi_t \cap B_{t+1}^{\complement} ] & \\ 
    & = \Pr[\cap_{i=1}^{t} B_i] - \Pr[B_{t+1}^{\complement}|\Pi_t]\Pr[\Pi_t] &\\
    & \geq \Pr[\cap_{i=1}^{t} B_i] - \Pr[B_{t+1}^{\complement}|\Pi_t]
\end{align*}
Where the first inequality follows from the fact that $\Pi_t$ is a supper set of $\cap_{i=1}^{t} B_i$ [Since, $\cap_{i=1}^{t} B_i\subseteq B_t \subseteq \Pi_t$].
Therefore from induction, for all $t\geq 1$ we get that: 
$$ \Pr\Bigg[\bigcap_{i=1}^{t} B_i \Bigg] \geq 1 - \sum_{i=1}^{t}\Pr\big[B_i^{\complement}|\Pi_{i-1}\big].$$
We get our required result by taking limit as $t$ goes to $\infty$.
\end{proof}

\corAgg*
\begin{proof}
Recall that from \cref{subsec:algconfidence} we know that $\Supp(\beta_k)\subseteq H_{k,t'}$ if $\beta_k\in \cset{k}{t'-1}$. For any arm $k$, using \cref{lem:probability-aggregation} with $B_t$ being the event that $\beta_k\in \cset{k}{t}$ and $\Pi_t$ being the event that $\Supp(\beta_k)\subseteq H_{k,t}$, we get that:
\begin{align*}
    &\Pr\Bigg[\bigcap_{t=1}^{\infty} \{\beta_k\in\cset{k}{t-1} \} \Bigg] \geq 1 - \sum_{t=1}^{\infty}\Pr[\beta_{k}\notin\cset{k}{t}|\Supp(\beta_{k})\subseteq H_{k,t}].
\end{align*}
Now, from union bound and the above inequality we get that:
\begin{align*}
    \Pr\Bigg[\bigcap_{k=1}^K\bigcap_{t=1}^{\infty} \{\beta_k\in\cset{k}{t-1} \} \Bigg] 
    &= 1 - \Pr\Bigg[\bigcup_{k=1}^K\bigg(\bigcap_{t=1}^{\infty} \{\beta_k\in\cset{k}{t-1} \}\bigg)^{\complement} \Bigg]\\
    & \geq 1 - \sum_{k=1}^K \Pr\Bigg[\bigg(\bigcap_{t=1}^{\infty} \{\beta_k\in\cset{k}{t-1} \}\bigg)^{\complement} \Bigg]\\
    & \geq 1 - \sum_{k=1}^K\sum_{t=1}^{\infty}\Pr[\beta_{k}\notin\cset{k}{t}|\Supp(\beta_{k})\subseteq H_{k,t}].
\end{align*}
\end{proof}

\section{Proofs for General Regret Analysis}
\label{app:general-regret-analysis}

We now re-state \cref{lem:general-regret-analysis} and provide the proof in following sub-sections.
\genregbound*
\subsection{Instantaneous Regret Decomposition} \label{subsec:instantaneous-regret-decomposition}

Let $r_t$ denote the instantaneous regret at time-step $t$. Suppose $\SurveyUCB$ picks arm~$i$ at time-step $t$, that is $\pi^a_t=i$. And, suppose arm~$j$ is the optimum arm at time-step $t$, that is $\pi^{a*}_t=j$. Now suppose $\pi^s_t$ denote the set of questions queried by $\SurveyUCB$, then 
$$\pi^s_t:=\bigcup_{k\in[K]}\Supp(\cset{k}{t-1}).$$ 
Hence we have that $\Supp(\cset{k}{t-1})\subseteq \pi^s_t$ for any arm $k\in[K]$ and time-step $t$. Therefore for any $\beta\in\cset{k}{t-1}$ and $z\in\R^d$, we have that $z^{\transpose}\beta = (z)_{\pi^t_s}^{\transpose}\beta$. Further from conditions stated in \cref{lem:general-regret-analysis}, we have that $\beta_k\in\cset{k}{t-1}$ for all $k\in[K]$ and $t\in[T]$. Therefore, we have:
\begin{align*}
    r_t &= X_t^{\transpose}\beta_j - X_t^{\transpose}\beta_i  \\
    &= (X_t)_{\pi^s_t}^{\transpose}\beta_j - (X_t)_{\pi^s_t}^{\transpose}\beta_i 
\end{align*}
Since $\SurveyUCB$ chooses arm $i=\pi^a_t\in \arg\max_{k\in[K]} \max_{\beta \in \cset{k}{t-1}} (X_t)_{\pi^s_t}^{\transpose}\beta$, and $\beta_j\in\cset{j}{t-1}$.
We have that $(X_t)_{\pi^s_t}^{\transpose}\optimisticEst_{i,t-1} \geq (X_t)_{\pi^s_t}^{\transpose}\beta_j$. Where $\optimisticEst_{i,t-1}$ denotes the optimistic estimate of arm~$i$'s parameter at time-step $t$, that is $\optimisticEst_{i,t-1} = \arg\max_{\beta\in \cset{i}{t-1}} (X_t)_{\pi^s_t}^{\transpose}\beta$. Therefore, we get that: 
\begin{align*}
    r_t & \leq (X_t)_{\pi^s_t}^{\transpose}\optimisticEst_{i,t-1} - (X_t)_{\pi^s_t}^{\transpose}\beta_i\\
    & = (X_t)_{\pi^s_t}^{\transpose}(\optimisticEst_{i,t-1} - \trueEst_{i,t-1}) + (X_t)_{\pi^s_t}^{\transpose}(\trueEst_{i,t-1} - \beta_i)
\end{align*}
Where $\trueEst_{i,t-1}$ denotes the estimate of arm~$i$'s parameter at time-step $t$. Now using Caushy-Schwarz inequality for weighted norm with respect to the positive definite matrix $\Vtmatrix:=\alpha I + \sum_{w\in S_{i,t-1}} (X_w)_{\pi^s_t} (X_w)_{\pi^s_t}^{\transpose}$, we get:
\begin{align*}
    r_t & \leq \|\optimisticEst_{i,t-1} - \trueEst_{i,t-1} \|_{\Vtmatrix} \|(X_t)_{\pi^s_t} \|_{\VtmatrixInv} + \|\trueEst_{i,t-1} - \beta_i \|_{\Vtmatrix}\|(X_t)_{\pi^s_t} \|_{\VtmatrixInv} \\
    & = \|\optimisticEst_{i,t-1} - \trueEst_{i,t-1} \|_{\Dtmatrix{i}{t-1}} \|(X_t)_{\pi^s_t} \|_{\VtmatrixInv} + \|\trueEst_{i,t-1} - \beta_i \|_{\Dtmatrix{i}{t-1}}\|(X_t)_{\pi^s_t} \|_{\VtmatrixInv} \\
    & \leq 2 \conftBound{i}{t-1} \| (X_t)_{\pi^s_t} \|_{\VtmatrixInv}
\end{align*}
Where the last inequality follows directly from the fact that our confidence sets are in Standard form. The first equality follows from the \cref{lem:simple-algebra}, the fact that $\beta_i,\trueEst_{i,t-1},\optimisticEst_{i,t-1}\in\cset{i}{t-1}$, and the fact that $\Dtmatrix{i}{t-1}=(\Vtmatrix)_{H_{i,t-1}}$ which follows from the fact that our confidence sets are in Standard form, where $H_{i,t-1}:=\Supp(\cset{i}{t-1})$. We defer the proof of \cref{lem:simple-algebra} to \cref{app:general-regret-help}. %
\begin{restatable}{observation}{simpleAlgebra}
\label{lem:simple-algebra}
Consider any vector $v\in\R^d$ any positive semi-definite matrix $A\in\R^{d\times d}$. Let $S$ be such that $\Supp(v)\subseteq S$. We have that $\|v\|_A = \|v\|_B$ if $B = (A)_S$.
\end{restatable}

Now since reward at any time $t$ lies in the range $[0,1]$. Therefore $r_t\in[0,2]$ for all $t$. Since $\conftBound{i}{t-1}\geq 1$, we further have that $r_t \leq 2 \conftBound{i}{t-1} \min \big\{ 1, \| (X_t)_{\pi^s_t} \|_{\VtmatrixInv}  \big\}$.

\subsection{Cumulative Regret}
From lemma~11 in \cite{abbasi2011improved}, we get:
\begin{lemma} [Abbasi-Yadkori, et al 1] \label{lem:abassi-regret-sum}
Let $\{X_t\}_{t=1}^{\infty}$ be a sequence in $\R^d$ and $V\in\R^{d\times d}$ that is positive definite. Define $V_t := V+\sum_{s=1}^t X_sX_s^{\transpose}$. Further if $\|X_t\|_2\leq L$ for all $t$, then:
\begin{align*}
    &\sum_{t=1}^n \min \big\{1, \|X_t\|^2_{V_{t-1}^{-1}} \big\} \leq 2\bigg( d\log \bigg(\frac{\trace(V)+nL^2}{d}\bigg) - \log\det(V) \bigg).
\end{align*}
\end{lemma}

Now, consider any arm $i \in [K]$, let $R_T^i$ denote the cumulative regret incurred by arm~$i$. Hence from Caushy-Schwarz inequality and from bound on $r_t$ in \cref{subsec:instantaneous-regret-decomposition}, we get that: %
\begin{align*}
    R_T^i & =  \sum_{t\in S_{i,T}} r_t \leq \sqrt{n_{i,T}\sum_{s\in S_{i,T}} r_t^2} \\
    & \leq \conftBound{i}{t-1}\sqrt{4n_{i,T} \sum_{s\in S_{i,T}} \min \big\{ 1, \| (X_t)_{\pi^s_t} \|^2_{\VtmatrixInv}  \big\}}
\end{align*}
From conditions stated in \cref{lem:general-regret-analysis} we have that $\|X_t\|_2\leq L$ for all $t$. Hence from \cref{lem:abassi-regret-sum} we get that:
$$ R_T^i \leq \conftBound{i}{t-1} \sqrt{8n_{i,T}\bigg( d\log \bigg(\frac{\trace(\alpha I)+n_{i,T}L^2}{d\alpha}\bigg)\bigg)} $$
Therefore, our total regret is of the form: 
\begin{align*}
     R_T &=\sum_{i\in[K]}R_T^i\\ 
     &\leq \sum_{i\in[K]} \conftBound{i}{T-1} \sqrt{8n_{i,T}\bigg( d\log \bigg(\frac{d\alpha+n_{i,T}L^2}{d\alpha}\bigg) \bigg)}.
\end{align*}
This completes the proof of \cref{lem:general-regret-analysis}.

\subsection{Proof for Observations}
\label{app:general-regret-help}
\simpleAlgebra*
\begin{proof}
Consider any $v\in\R^d$ any $A\in\R^{d\times d}$. Let $S$ be such that $\Supp(v)\subseteq S \subseteq [d]$. Let $B=(A)_S$, that is we get $B$ by setting rows and columns of $A$ not in $S$ to zero. Consider $A'$, which we get by setting columns of $A$ not in $S$ to zero. Since rows of $A'$ have support in $S\supset \Supp(v)$ and are the same as the $A$'s rows within the support, we have that $(A-A')v=0$. That is, $Av=A'v$. Now note that we can get $B$ by setting rows of $A'$ not in $S$ to zero. Similarly, we get $v^{\transpose}A'=v^{\transpose}B$. Therefore, we have that:
$$ \|v\|_A^2 = v^{\transpose}Av = v^{\transpose}A'v = v^{\transpose} B v = \|v\|_B^2. $$
Since $A$ and hence $B$ are psd, this also gives us that $\|v\|_A=\|v\|_B$.
\end{proof}

\section{Proofs for \cref{lem:adapted-elnet-tail}}
\label{app:elnet-tail-inequality}

Consider a linear model $Y=\X\beta + \epsilon$, with design matrix $\X\in\R^{n\times d}$, response vector $Y\in\R^n$, and noise vector $\epsilon\in\R^n$. Where $\epsilon_{t}$ are independent sequence of $\sigma$-sub-Gaussian random variables. Now, from lemma~{EC2} in \cite{bastani2015online}, we have that:

\begin{lemma}[Bastani and Bayati] 
\label{lem:adaptive-event-bastani}
Let $X_t$ denote the $t$-th row of $\X$. Let $Y(t)$ denote the $t$-th entry of $Y$. The sequence $\{X_t| t=1,2,\dots,n \}$ form an adapted sequence of observations. That is, $X_t$ may depend on $\{X_{t'}, Y(t') \}_{t'=1}^{t-1}$. Also assume all realizations of $X_t$ satisfy $\|X_t\|_{\infty} \leq L$. Now, define the event: 
$$ \mathcal{F}(\lambda_0(\gamma)) := \bigg\{\max_{r\in[d]} (2|\epsilon^{\transpose}X^r|/n) \leq \lambda_0(\gamma) \bigg\}. $$
Where $X^r$ is the $r$-th column of $\X$ and $\lambda_0(\gamma):=2\sigma L \sqrt{(\gamma^2+2\log d)/n}$. Then, we have $\Pr[\F(\lambda_0(\gamma))] \geq 1 - 2\exp[-\gamma^2/2]$.
\end{lemma}

We now state a useful basic inequality for Elastic net estimators. The proof is similar to basic inequalities proved for Lasso in \cite{buhlmann2011statistics} and defer the proof to \cref{sec:basic-ineq-elnet}.

\begin{restatable}[Basic inequality for Elastic net]{lemma}{elnetBasic}\label{lem:basic-ineq-elnet}
Consider a linear model $Y=\X\beta + \epsilon$, with design matrix $X\in\R^{n\times d}$, response vector $Y\in\R^n$, and noise vector $\epsilon\in\R^n$. We then have that:
$$ \frac{1}{n}\|\trueEst - \beta \|_D^2 + \lambda\|\trueEst\|_1 \leq \frac{2}{n}\epsilon^{\transpose}\X(\trueEst-\beta)+\lambda\|\beta\|_1 + \frac{4\alpha}{n}\|\beta\|_2^2. $$
Where $\trueEst$ is the Elastic net estimate with parameters $\alpha,\lambda \geq 0$. And, $D=\X^{\transpose}\X + \alpha I$.
\end{restatable}

We define the following lemma is result of simplifying \cref{lem:basic-ineq-elnet} under the high-probability event $\F(\lambda_0)$ given in \cref{lem:adaptive-event-bastani} when $\lambda\geq 2\lambda_0$. 

\begin{lemma}
\label{lem:basic-elnet-on-Bastani-event}
Consider a linear model $Y=\X\beta + \epsilon$, with design matrix $X\in\R^{n\times d}$, response vector $Y\in\R^n$, and noise vector $\epsilon\in\R^n$. Let $\trueEst$ be the Elastic net estimate with parameters $\alpha,\lambda \geq 0$ and let $D=\X^{\transpose}\X+\alpha I$. When $\lambda\geq 2\lambda_0$ and $\F(\lambda_0)$ holds, we have that:
$$ \frac{2}{n}\|\trueEst - \beta \|_D^2 \leq 3\lambda\|\beta\|_1 + \frac{8\alpha}{n}\|\beta\|_2^2. $$
\end{lemma}
\begin{proof}
Since $\F(\lambda_0)$ holds and $\lambda\geq 2\lambda_0$, we get:
\begin{align*}
    \frac{2}{n}\epsilon^{\transpose}\X(\trueEst-\beta) &\leq \frac{1}{n}\Big(\max_{r\in [d]} 2|\epsilon^{\transpose}X^r| \Big)\|\trueEst-\beta\|_1 \\
    & \leq \lambda_0\|\trueEst - \beta\|_1 \\
    & \leq \frac{\lambda}{2}\|\trueEst - \beta\|_1 
\end{align*}
Hence from \cref{lem:basic-ineq-elnet} and the above inequality, we have that:
\begin{align*}
    & \frac{1}{n}\|\trueEst - \beta \|_D^2 + \lambda\|\trueEst\|_1 \leq \frac{2}{n}\epsilon^{\transpose}X(\trueEst-\beta)+\lambda\|\beta\|_1 + \frac{4\alpha}{n}\|\beta\|_2^2 \\
    \implies & \frac{2}{n}\|\trueEst - \beta \|_D^2 \leq \lambda\|\trueEst-\beta\|_1+2\lambda\|\beta\|_1+\frac{8\alpha}{n}\|\beta\|_2^2 - 2 \lambda\|\trueEst\|_1 
\end{align*}
Now since $\|\trueEst\|_1 \geq 0 $, we get:
\begin{align*}
    \implies & \frac{2}{n}\|\trueEst - \beta \|_D^2 \leq \lambda(\|\trueEst-\beta\|_1 - \|\trueEst\|_1 )+2\lambda\|\beta\|_1+\frac{8\alpha}{n}\|\beta\|_2^2 \\
    \implies & \frac{2}{n}\|\trueEst - \beta \|_D^2 \leq 3\lambda\|\beta\|_1+\frac{8\alpha}{n}\|\beta\|_2^2  
\end{align*}
Where the last implication follows from triangle inequality.
\end{proof}

Combining \cref{lem:adaptive-event-bastani} and \cref{lem:basic-elnet-on-Bastani-event}, we get \cref{lem:adapted-elnet-tail}:

\elnetTail*
\begin{proof}
Let $\lambda_0:=2\sigma L \sqrt{(\gamma^2+2\log d)/n}$. Now note that $\lambda \geq 2\lambda_0$. Therefore from \cref{lem:adaptive-event-bastani} and \cref{lem:basic-elnet-on-Bastani-event}, we get that with probability at least $1-2\exp[-\gamma^2/2]$, we have:
\begin{align*}
    & \frac{2}{n}\|\trueEst - \beta \|_D^2 \leq 3\lambda\|\beta\|_1 + \frac{8\alpha}{n}\|\beta\|_2^2 \\
    \implies & \|\trueEst - \beta \|_D^2 \leq \frac{n}{2} 3 \lambda b + 4\alpha b^2\\
    \implies & \|\trueEst - \beta \|_D^2 \leq 6\sigma L b \sqrt{n(\gamma^2+2\log(d))} + 4\alpha b^2
\end{align*}
Where the first implication follows from $\|\beta\|_2\leq \|\beta\|_1\leq b$. And, the last implication follows from our choise of $\lambda$.
\end{proof}

\section{Basic Inequality for Elastic Net}
\label{sec:basic-ineq-elnet}
\elnetBasic*
\begin{proof}
Since $\trueEst$ is the Elastic net estimate for parameters $\alpha,\lambda>0$, we have that:
\begin{align*}
    & \frac{1}{n}\Big[\|Y-\X\trueEst\|_2^2 + 2\alpha\|\trueEst\|_2^2 \Big] + \lambda\|\trueEst\|_1 \leq \frac{1}{n}\Big[\|Y-\X\beta\|_2^2 + 2\alpha\|\beta\|_2^2 \Big] + \lambda\|\beta\|_1\\
    \implies & \frac{1}{n}\|\X\trueEst\|_2^2 -\frac{2}{n}Y^{\transpose}\X\trueEst + \frac{2\alpha}{n}\|\trueEst\|_2^2 + \lambda\|\trueEst\|_1 \leq \frac{1}{n}\|\X\beta\|_2^2 -\frac{2}{n}Y^{\transpose}\X\beta + \frac{2\alpha}{n}\|\beta\|_2^2 + \lambda\|\beta\|_1
\end{align*}
Adding $\frac{1}{n}\|\X\beta\|_2^2 - \frac{2}{n}(\X\beta)^{\transpose}(\X\trueEst)+\frac{2\alpha}{n}\|\beta\|_2^2$ to both sides, and re-arranging terms we get:
\begin{equation}
\label{eq:large-basic-elnet-ineq}
\begin{split}
&\frac{1}{n}\|\X\trueEst\|_2^2+\frac{1}{n}\|\X\beta\|_2^2 - \frac{2}{n}(\X\beta)^{\transpose}(\X\trueEst) + \frac{2\alpha}{n}(\|\trueEst \|_2^2 + \|\beta \|_2^2) + \lambda\|\trueEst\|_1 \\
 \leq & \Big[\frac{2}{n} Y^{\transpose}(\X\trueEst) - \frac{2}{n}(\X\beta)^{\transpose}(\X\trueEst) \Big] + \Big[ \frac{2}{n}\|\X\beta\|_2^2 - \frac{2}{n}Y^{\transpose}(\X\beta) \Big] + \frac{4\alpha}{n}\|\beta\|_2^2 + \lambda\|\beta\|_1.
\end{split}
\end{equation}
Note that we can lower bound the LHS of \cref{eq:large-basic-elnet-ineq}:
\begin{equation}
\label{eq:label-basic-elnet-ineq-LHS}
\begin{split}
    &\frac{1}{n}\|\X\trueEst\|_2^2+\frac{1}{n}\|\X\beta\|_2^2 - \frac{2}{n}(\X\beta)^{\transpose}(\X\trueEst) + \frac{2\alpha}{n}(\|\trueEst \|_2^2 + \|\beta \|_2^2) + \lambda\|\trueEst\|_1\\
    = & \frac{1}{n}\|\X(\trueEst-\beta)\|_2^2 + \frac{\alpha}{n}(\|\trueEst - \beta \|_2^2 + \|\trueEst + \beta \|_2^2) + \lambda\|\trueEst\|_1\\
    \geq & \frac{1}{n}\|\X(\trueEst-\beta)\|_2^2 + \frac{\alpha}{n}\|\trueEst - \beta \|_2^2 + \lambda\|\trueEst\|_1\\
    = & \frac{1}{n}\|\trueEst-\beta\|_D^2 + \lambda\|\trueEst\|_1
\end{split}
\end{equation}
By substituting $Y=\X\beta + \epsilon$, we simplify RHS of \cref{eq:large-basic-elnet-ineq} and get:
\begin{equation}
\label{eq:large-basic-elnet-ineq-RHS}
\begin{split}
&\Big[\frac{2}{n} Y^{\transpose}(\X\trueEst) - \frac{2}{n}(\X\beta)^{\transpose}(\X\trueEst) \Big] + \Big[ \frac{2}{n}\|\X\beta\|_2^2 - \frac{2}{n}Y^{\transpose}(\X\beta) \Big] + \frac{4\alpha}{n}\|\beta\|_2^2 + \lambda\|\beta\|_1\\
 = & \Big[\frac{2}{n} \epsilon^{\transpose}(\X\trueEst) \Big] + \Big[ - \frac{2}{n}\epsilon^{\transpose}(\X\beta) \Big] + \frac{4\alpha}{n}\|\beta\|_2^2 + \lambda\|\beta\|_1\\
 = & \frac{2}{n}\epsilon^{\transpose}\X(\trueEst-\beta) + \frac{4\alpha}{n}\|\beta\|_2^2 + \lambda\|\beta\|_1.
\end{split}
\end{equation}
Putting \cref{eq:large-basic-elnet-ineq}, \cref{eq:label-basic-elnet-ineq-LHS}, and \cref{eq:large-basic-elnet-ineq-RHS} together, we get:
$$ \frac{1}{n}\|\trueEst - \beta \|_D^2 + \lambda\|\trueEst\|_1 \leq \frac{2}{n}\epsilon^{\transpose}\X(\trueEst-\beta)+\lambda\|\beta\|_1 + \frac{4\alpha}{n}\|\beta\|_2^2. $$
\end{proof}

\section{Proof for \cref{thm:ridge-survey-ucb-regret}}
\label{app:ridge-regret}

\ridgeRegret*
\begin{proof}
Recall that in Ridge $\SurveyUCB$, for the construction of our confidence sets, for all $i\in[K]$ and $t\in[T]$, we choose:
\begin{align*}
    \conftBound{k}{t} = \sigma \sqrt{|H_{k,t}|\log\bigg(\frac{1+n_{k,t}L^2/\alpha}{\delta/(K (1+n_{k,t})^2)}\bigg)} + \sqrt{\alpha}b
\end{align*}
Hence from \cref{cor:prob-aggregation} and \cref{lem:adapted-elnet-tail}, we have that:
$$\Pr\Big[\forall i\in[K], \forall t, \beta_i\in\cset{i}{t-1} \Big] \geq 1-\delta .$$
Therefore, with probability at least $1-\delta$, we get that from \cref{lem:general-regret-analysis} the following regret guarantee holds:
\begin{align*}
    & R_T = \sum_{i\in[K]}R_T^i\\ 
    & \leq \sum_{i\in[K]} \conftBound{i}{T-1} \sqrt{8dn_{i,T}\log \bigg(\frac{d\alpha+n_{i,T}L^2}{d\alpha}\bigg)}\\
    & \leq \sum_{i\in[K]} \bigg(\sigma \sqrt{d\log\bigg(\frac{1+TL^2/\alpha}{\delta/(K (1+T)^2)}\bigg)} + \sqrt{\alpha}b\bigg) \sqrt{8dn_{i,T}\log \bigg(\frac{d\alpha+TL^2}{d\alpha}\bigg)}\\
    & \leq \bigg(\sigma \sqrt{d\log\bigg(\frac{1+TL^2/\alpha}{\delta/(K (1+T)^2)}\bigg)} + \sqrt{\alpha}b\bigg) \sqrt{8d\log \bigg(\frac{d\alpha+TL^2}{d\alpha}\bigg)} \sum_{i\in[K]} \sqrt{n_{i,T}}\\
    & \leq \bigg(\sigma \sqrt{d\log\bigg(\frac{1+TL^2/\alpha}{\delta/(K (1+T)^2)}\bigg)} + \sqrt{\alpha}b\bigg) \sqrt{8d\log \bigg(\frac{d\alpha+TL^2}{d\alpha}\bigg)} \sqrt{TK}
\end{align*}
Where the last inequality follows from Caushy-Schwarz inequality and the fact that $T=\sum_{i=1}^K n_{i,T}$. Therefore our high-probability regret bound is $O(d \sqrt{KT\log(K)} \log(T))$.
\end{proof}

\section{Proof for \cref{thm:elastic-net-survey-ucb-regret}}
\label{app:elnet-regret}
\elnetRegret*
\begin{proof}
Recall that in Elastic Net $\SurveyUCB$ we use Elastic net regression to estimate $\trueEst_{i,t}$, with regularization parameters $\alpha$ and $\lambda_{i,t}$ for all arms $i\in[K]$ and time-steps $t$. Where:
$$ \lambda_{i,t} := 4\sigma L \sqrt{\frac{2}{n_{i,t}} \log\bigg(\frac{4dKn_{i,t}^2}{\delta} \bigg)},\text{ } \forall i\in[K], \forall t\in[T]. $$
And for the construction of our confidence sets, for all $i\in[K]$ and $t\in[T]$, we choose:
$$\conftBound{i}{t} := \sqrt{6\sigma L b\sqrt{2n_{i,t}\log\bigg(\frac{4dKn_{i,t}^2}{\delta}\bigg)} + 4\alpha b^2}. $$
Now from \cref{cor:prob-aggregation} and \cref{lem:adapted-elnet-tail}, we have that:
$$\Pr\Big[\forall i\in[K], \forall t, \beta_i\in\cset{i}{t-1} \Big] \geq 1-\delta .$$
Therefore, with probability at least $1-\delta$, we get that from \cref{lem:general-regret-analysis} the following regret guarantee holds:
\begin{align*}
    & R_T = \sum_{i\in[K]}R_T^i\\ 
    & \leq \sum_{i\in[K]} \conftBound{i}{T-1} \sqrt{8dn_{i,T}\log \bigg(\frac{d\alpha+n_{i,T}L^2}{d\alpha}\bigg)}\\
    & = \sum_{i\in[K]} \sqrt{6\sigma L b\sqrt{2n_{i,t}\log\bigg(\frac{4dKn_{i,T}^2}{\delta}\bigg)} + 4\alpha b^2} \sqrt{8dn_{i,T}\log \bigg(\frac{d\alpha+n_{i,T}L^2}{d\alpha}\bigg)}\\
    & \leq \sum_{i\in[K]} 4 n_{i,T}^{3/4} \sqrt{3\sigma Lb d \bigg(\sqrt{2\log\bigg(\frac{4dKT^2}{\delta}\bigg)} + \frac{2\alpha b}{3\sigma L} \bigg) \log \bigg(\frac{d\alpha+TL^2}{d\alpha}\bigg)}\\
    & \leq 4 T^{3/4} K^{1/4} d^{1/2} \sqrt{3\sigma Lb \bigg(\sqrt{2\log\bigg(\frac{4dKT^2}{\delta}\bigg)} + \frac{2\alpha b}{3\sigma L} \bigg) \log \bigg(\frac{d\alpha+TL^2}{d\alpha}\bigg)}
\end{align*}
Where the last inequality follows from Holder's inequality and the fact that $T=\sum_{i\in[K]}n_{i,T}$.
Therefore, our high-probability regret bound is $O(K^{1/4} d^{1/2} T^{3/4}  \log^{3/4}(T) \log^{1/4}(dK) )$.
\end{proof}

\section{Implementation Issues}
Here we describe how one can handle some of the difficulties in implementing our algorithms.

\subsection{Ridge $\SurveyUCB$}
We need to add a stabilizing numerical approximations that sets numbers between $[-10^{-8}, 10^{-8}]$ to zero throughout the algorithm. We do this for all parameters that we store. The reason we do this is because, we often need to take the pseudo-inverse of $\Dtmatrix{i}{t-1}$ for different arms at every time-step. This makes our updates particularly vulnerable to small errors around zero, and leads to unexpected behavior without it.

\subsection{Elastic Net $\SurveyUCB$}
The python sklearn implementation for Elastic Net regression seems to have some issues with convergence for large L1 regularization. For smaller dimension problems, this can be fixed by setting the tolerance parameter to be very low (which also increases the running time). For larger dimension problems, it seems better to just divide the L1 regularization parameter by the dimension of the problem and appropriately adjust the confidence set bounds. In particular, we choose:
\begin{align*}
    \lambda_{i,t} := \frac{1}{d} 4\sigma L \sqrt{\frac{2}{n_{i,t}} \log\bigg(\frac{4dKn_{i,t}^2}{\delta} \bigg)},\text{ } \forall i\in[K], \forall t\in[T].
\end{align*}
Note that we are just dividing our original L1 regularization parameter by a factor of $d$. To ensure that our confidence sets hold with high-probability, we choose for all arms and time-steps: \footnote{Easy to check, simple modification of analysis in \cref{lem:basic-elnet-on-Bastani-event}.}
\begin{align*}
    \conftBound{i}{t} := \sqrt{6\sigma L (b + \|\trueEst_{i,t}\| ) d\sqrt{2n_{i,t}\log\bigg(\frac{4dKn_{i,t}^2}{\delta}\bigg) } + 4\alpha b^2}.
\end{align*}

\subsection{Interactive Surveys}
\label{app:interactive-survey}
The only issue with implementing \cref{alg:interactive-survey} is that the optimization problem in \cref{eq:optimization-interactive} is non-convex. It is worth noting that any upper bound to the optimization problem would maintain the regret performance of interactive $\SurveyUCB$. And better upper-bounds, would be able to demonstrate better savings in terms of survey length. Hence one could use reasonable convex relaxations for this optimization problem in \cref{alg:interactive-survey}. 

In this sub-section, we provide a well performing heuristic as a surrogate to the optimization problem when the context-space is a $[0,1]^d$. Now note that:
\begin{align*}
    \max_{x\in F} \quad &  x^{\transpose}\trueEst_{w,t-1}+\sqrt{\conftBound{w}{t-1}(x^{\transpose}(\Dtmatrix{w}{t-1})^{-1}x)}\\
    \leq \max_{x,y\in F} \quad &  x^{\transpose}\trueEst_{w,t-1}+\sqrt{\conftBound{w}{t-1}(y^{\transpose}(\Dtmatrix{w}{t-1})^{-1}y)}
\end{align*}
Note that the optimization problem $\max_{x\in F} x^{\transpose}\trueEst_{w,t-1}$ is convex and infact linear when our context-space can be represented by linear constraints. We now only need to optimize the non-convex problem:
\begin{align}
\label{eq:uncertainity-maximization}
    \max_{y\in F} \sqrt{\conftBound{w}{t-1}(y^{\transpose}(\Dtmatrix{w}{t-1})^{-1}y)} = \sqrt{\max_{y\in F}\conftBound{w}{t-1}(y^{\transpose}(\Dtmatrix{w}{t-1})^{-1}y)}
\end{align}
This is a non-convex quadratic optimization problem, where $(\Dtmatrix{w}{t-1})^{-1}$ is symmetric (and positive semi-definite). This problem has been well studied and has good approximations based on SDP relaxations \cite{ye1999approximating}. 

For our simulations we use the following simple heuristic as a surrogate to this optimization problem (\cref{eq:uncertainity-maximization}) that has worked surprisingly well. In the above problem (\cref{eq:uncertainity-maximization}), we want maximize uncertainty in reward over contexts that are consistent with our queries. Intuitively, uncertainty is maximized when the context $y$ is as large as possible. When our context space is $[0,1]^d$, we can simply choose $y$ that is largest co-ordinate wise. That is, $y_j=1$ if $j$ is not queried (i.e. $j\notin U$) and $y_j=(X_t)_j$ if $j$ has been queried (i.e. $j\in U$).

\end{document}